\documentclass{article}

\usepackage{microtype}
\usepackage{graphicx}
\usepackage{subcaption}
\usepackage{colortbl}
\usepackage{xcolor} 
\usepackage{booktabs} 
\usepackage{hyperref}

\usepackage{xcolor}

\usepackage[preprint]{icml2026}

\makeatletter
\@ifundefined{algorithm}{}{%
  \let\algorithm\@undefined
  \let\endalgorithm\@undefined
}
\@ifundefined{algorithm*}{}{%
  \expandafter\let\csname algorithm*\endcsname\@undefined
  \expandafter\let\csname endalgorithm*\endcsname\@undefined
}
\@ifundefined{listofalgorithms}{}{%
  \let\listofalgorithms\@undefined
}
\makeatother

\usepackage[ruled,vlined,linesnumbered]{algorithm2e}

\usepackage{amsthm}
\usepackage{amsmath,amssymb,mathtools}
\usepackage[most]{tcolorbox}

\usepackage[capitalize,noabbrev]{cleveref}
\usepackage[table,dvipsnames]{xcolor} 
\usepackage{colortbl}
\newcommand{\gain}[1]{\fontsize{10pt}{6pt}\textbf{\selectfont\textcolor{red}{${}_{\uparrow #1}$}}}
\newcommand{\loss}[1]{\fontsize{10pt}{6pt}\selectfont\textcolor{NavyBlue}{${}_{\downarrow #1}$}}
\theoremstyle{plain}
\newtheorem{theorem}{Theorem}[section]

\newtheorem{lemma}[theorem]{Lemma}

\theoremstyle{definition}

\theoremstyle{remark}

\usepackage{multirow}
\usepackage{multicol}
\usepackage{graphicx}
\usepackage{booktabs}
\usepackage[most]{tcolorbox}
\usepackage{enumitem}
\usepackage{xcolor}
\usepackage[textsize=tiny]{todonotes}

\usepackage[most]{tcolorbox}
\usepackage{xcolor}
\definecolor{lightgray}{gray}{0.92}

\providecommand{\roleSystem}{\textbf{\textcolor{red!70!black}{[System]:}} }
\providecommand{\roleUser}{\textbf{\textcolor{blue!70!black}{[User]:}} }
\providecommand{\promptVar}[1]{\textcolor{violet}{\texttt{\{#1\}}}}

\newtcolorbox[auto counter, number within=section]{promptbox}[2]{
  colback=gray!5!white,
  colframe=gray!75!black,
  fonttitle=\bfseries,
  title={Prompt \thetcbcounter: #1}, 
  label={#2}, 
  enhanced,
  attach boxed title to top left={yshift=-2mm, xshift=2mm},
  boxed title style={colback=gray!75!black},
  breakable,     
  lines before break=25,
  left=2mm, right=2mm, top=2mm, bottom=2mm
}

\icmltitlerunning{UMEM: Unified Memory Extraction and Management Framework for Generalizable Memory}

\begin{document}

\twocolumn[
  \icmltitle{UMEM: Unified Memory Extraction and Management Framework for Generalizable Memory}



  \icmlsetsymbol{equal}{*}

  \begin{icmlauthorlist}
    \icmlauthor{Yongshi Ye}{xmu}
    \icmlauthor{Hui Jiang}{}
    \icmlauthor{Feihu Jiang}{alibaba}
    \icmlauthor{Tian Lan}{alibaba}
    \icmlauthor{Yichao Du}{alibaba}
    \icmlauthor{Biao Fu}{tongyi}
    \icmlauthor{Xiaodong Shi}{xmu}
    \icmlauthor{Qianghuai Jia}{alibaba}
    \icmlauthor{Longyue Wang}{alibaba}
    \icmlauthor{Weihua Luo}{alibaba}
  \end{icmlauthorlist}

  \icmlaffiliation{xmu}{Xiamen University}
  \icmlaffiliation{alibaba}{Alibaba International Digital Commerce}
  \icmlaffiliation{tongyi}{Tongyi Lab, Alibaba Group}

  \icmlcorrespondingauthor{Longyue Wang}{wanglongyue.wly@alibabainc.com}

  \icmlkeywords{Machine Learning, ICML}

  \vskip 0.3in
]



\printAffiliationsAndNotice{}  

\begin{abstract}
Self-evolving memory serves as the trainable parameters for Large Language Models (LLMs)-based agents, where extraction (distilling insights from experience) and management (updating the memory bank) must be tightly coordinated. 
Existing methods predominately optimize memory management while treating memory extraction as a static process, resulting in poor generalization, where agents accumulate instance-specific noise rather than robust memories.
To address this, we propose \textbf{U}nified \textbf{M}emory \textbf{E}xtraction and \textbf{M}anagement (UMEM), a self-evolving agent framework that jointly optimizes a Large Language Model to simultaneous extract and manage memories.
To mitigate overfitting to specific instances, we introduce Semantic Neighborhood Modeling and optimize the model with a neighborhood-level marginal utility reward via GRPO.
This approach ensures memory generalizability by evaluating memory utility across clusters of semantically related queries.
Extensive experiments across five benchmarks demonstrate that UMEM significantly outperforms highly competitive baselines, achieving up to a 10.67\% improvement in multi-turn interactive tasks. Futhermore, UMEM maintains a monotonic growth curve during continuous evolution. 
Codes and models will be publicly released.
\end{abstract}
\definecolor{lightpurple1}{rgb}{1.00, 1.00, 1.00}
\definecolor{lightpurple2}{rgb}{0.98, 0.98, 0.99}
\definecolor{lightpurple3}{rgb}{0.96, 0.96, 0.98}
\definecolor{lightpurple4}{rgb}{0.94, 0.93, 0.97}
\definecolor{lightpurple5}{rgb}{0.91, 0.89, 0.96}
\definecolor{lightpurple6}{rgb}{0.89, 0.85, 0.94}
\definecolor{lightpurple7}{rgb}{0.86, 0.80, 0.92}

\definecolor{lightblue1}{rgb}{1.00, 1.00, 1.00}
\definecolor{lightblue2}{rgb}{0.96, 0.99, 1.00}
\definecolor{lightblue3}{rgb}{0.92, 0.97, 1.00}
\definecolor{lightblue4}{rgb}{0.88, 0.95, 1.00}
\definecolor{lightblue5}{rgb}{0.83, 0.93, 1.00}
\definecolor{lightblue6}{rgb}{0.77, 0.90, 1.00}
\definecolor{lightblue7}{rgb}{0.70, 0.86, 1.00}  

\definecolor{tealgreen1}{rgb}{1.00, 1.00, 1.00}
\definecolor{tealgreen2}{rgb}{0.94, 0.98, 0.97}
\definecolor{tealgreen3}{rgb}{0.88, 0.97, 0.94}
\definecolor{tealgreen4}{rgb}{0.82, 0.96, 0.91}
\definecolor{tealgreen5}{rgb}{0.76, 0.95, 0.88}
\definecolor{tealgreen6}{rgb}{0.69, 0.94, 0.85}
\definecolor{tealgreen7}{rgb}{0.61, 0.91, 0.85}  

\definecolor{lightpink1}{rgb}{1.00, 0.98, 0.99}
\definecolor{lightpink2}{rgb}{1.00, 0.96, 0.98}
\definecolor{lightpink3}{rgb}{1.00, 0.94, 0.97}
\definecolor{lightpink4}{rgb}{1.00, 0.92, 0.96}
\definecolor{lightpink5}{rgb}{1.00, 0.90, 0.95}
\definecolor{lightpink6}{rgb}{1.00, 0.88, 0.94}
\definecolor{lightpink7}{rgb}{1.00, 0.85, 0.93}  

\definecolor{lightcyan1}{rgb}{1.00, 1.00, 1.00}
\definecolor{lightcyan2}{rgb}{0.97, 0.99, 0.99}
\definecolor{lightcyan3}{rgb}{0.92, 0.98, 0.98}
\definecolor{lightcyan4}{rgb}{0.84, 0.95, 0.96}
\definecolor{lightcyan5}{rgb}{0.76, 0.91, 0.94}
\definecolor{lightcyan6}{rgb}{0.68, 0.87, 0.92}
\definecolor{lightcyan7}{rgb}{0.60, 0.83, 0.90}


\definecolor{pinkpurple1}{rgb}{1.00, 1.00, 1.00}  
\definecolor{pinkpurple2}{rgb}{0.98, 0.95, 0.98}
\definecolor{pinkpurple3}{rgb}{0.96, 0.90, 0.96}
\definecolor{pinkpurple4}{rgb}{0.93, 0.85, 0.94}
\definecolor{pinkpurple5}{rgb}{0.89, 0.77, 0.91}
\definecolor{pinkpurple6}{rgb}{0.85, 0.69, 0.88}
\definecolor{pinkpurple7}{rgb}{0.80, 0.60, 0.85}

\definecolor{peach1}{rgb}{1.00, 0.99, 0.98}
\definecolor{peach2}{rgb}{1.00, 0.97, 0.95}
\definecolor{peach3}{rgb}{1.00, 0.95, 0.92}
\definecolor{peach4}{rgb}{1.00, 0.93, 0.89}
\definecolor{peach5}{rgb}{1.00, 0.90, 0.86}
\definecolor{peach6}{rgb}{1.00, 0.87, 0.83}
\definecolor{peach7}{rgb}{1.00, 0.84, 0.80} 

\definecolor{pinkgrad1}{RGB}{255, 255, 255}  
\definecolor{pinkgrad2}{RGB}{254, 248, 250}  
\definecolor{pinkgrad3}{RGB}{254, 241, 245}  
\definecolor{pinkgrad4}{RGB}{253, 234, 240}  
\definecolor{pinkgrad5}{RGB}{253, 226, 235}  
\definecolor{pinkgrad6}{RGB}{252, 218, 230}  
\definecolor{pinkgrad7}{RGB}{250, 202, 220}  

\section{Introduction}
\begin{figure*}[th]  
    \centering
    \includegraphics[width=0.92\textwidth]{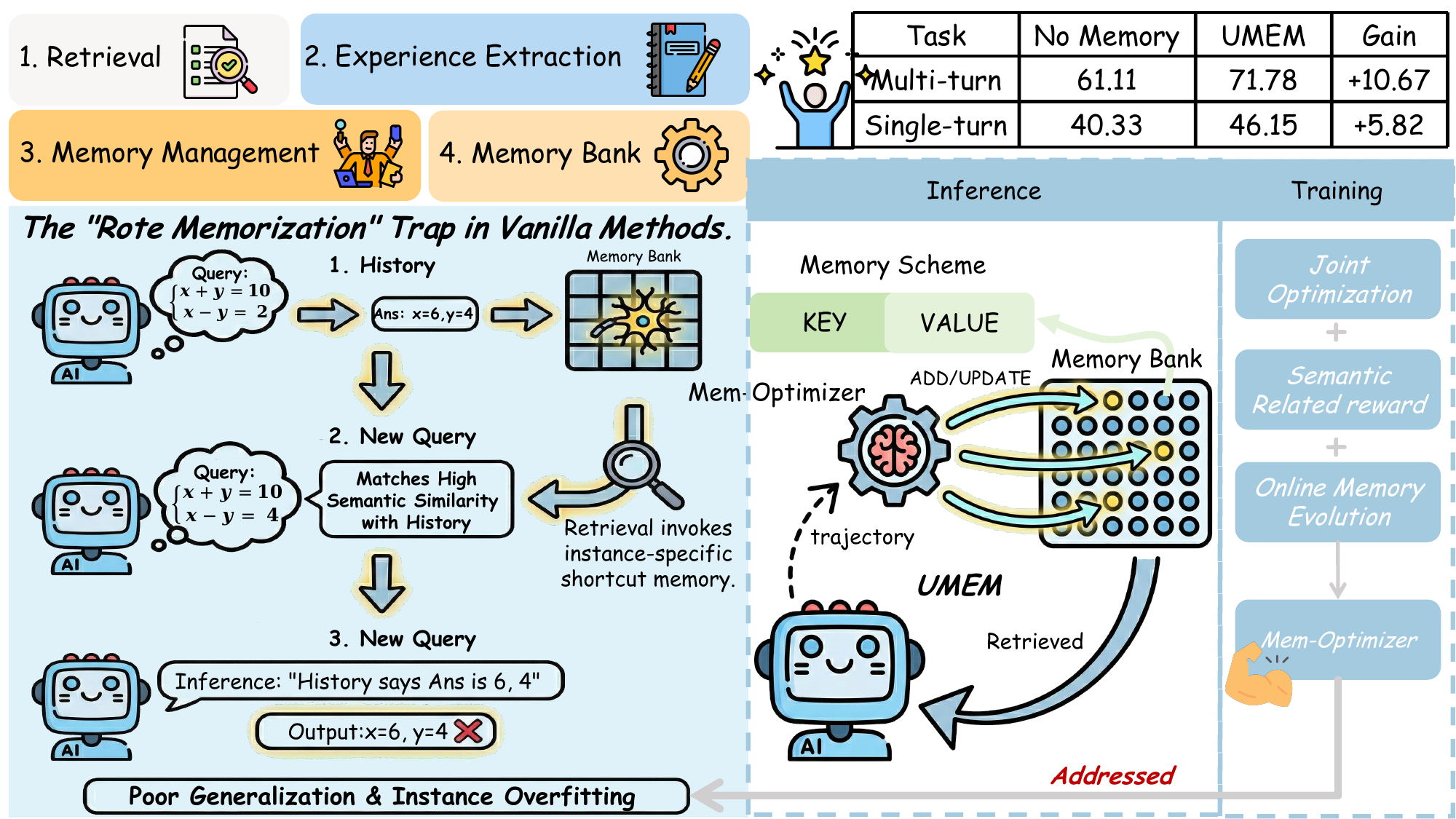}
    \caption{Comparison between the conventional memory pipeline and our proposed UMEM framework. \textbf{Left}: Vanilla methods suffer from the "Rote Memorization" trap, overfitting to instance-specific noise. \textbf{Right}: UMEM utilizes a learnable Mem-Optimizer to jointly optimize extraction and management. This distills generalizable principles, ensuring robust performance and avoiding noise accumulation.}
    \label{fig:overview}
\end{figure*}

Self-evolution is a fundamental capability for agents operating in dynamic, open-ended environments~\cite{zhang2026memrlselfevolvingagentsruntime}. While Large Language Models (LLMs) serve as powerful backbones for agents, their parameters typically remain frozen after deployment, limiting their ability to learn from continuous interactions. To overcome this limitation, long-term memory serves as trainable parameters of agents that can be updated from online experience~\cite{cai2025flexcontinuousagentevolution,cai2025trainingfreegrouprelativepolicy,ouyang2025reasoningbankscalingagentselfevolving,wei2025evomemorybenchmarkingllmagent}. 
Conceptually, a self-evolving agent system mirrors the neural network optimization~\cite{Rumelhart1986LearningRB,cai2025flexcontinuousagentevolution,ouyang2025reasoningbankscalingagentselfevolving}: 
(1) Forward Pass: the frozen agent executes a task given retrieved memories from memory bank; and
(2) Backward Optimization: a memory optimizer extracts insights (memories) from the experience and consolidate them into the memory bank~\cite{xu2025amemagenticmemoryllm,yan2026memoryr1enhancinglargelanguage}.
Therefore, the bottleneck of the self-evolving agent lies in the capability of this memory optimizer.


While numerous works have improved the memory optimizer, they predominantly focus on memory management, treating extraction as a static process via prompting off-the-shelf LLMs~\cite{wu2025evolverselfevolvingllmagents,yan2026memoryr1enhancinglargelanguage,fang2026mempexploringagentprocedural}, without optimizing for explicit generalization.
Consequently, self-evolving agents suffer from two critical problems:
(1) \textbf{Accumulation of Instance-Specific Noise}: As shown in Figure~\ref{fig:overview}, static memory extraction blindly retains instance-specific details rather than generalizable principles~\cite{qin2024o1replicationjourneystrategic}, causing progressive memory pollution and poor generalization; 
(2) \textbf{Management Misalignment}: The extracted memories are often inconsistent with the corresponding management policy, rendering even an optimal management policy ineffective. 
Therefore, even an well-optimized management policy cannot compensate for low-quality extracted memories, undermining both task performance and cross-task generalization of the self-evolving agents.


To bridge this gap, we propose \textbf{Unified Memory Extraction and Management (UMEM)}, a self-evolving agent framework that jointly optimizes the memory extraction and management capability of memory optimizer.
Structurally, UMEM consists of three primary components: a frozen Agent Executor (inference engine), a Memory Bank (the external parameters of self-evolving agents), and a learned memory optimizer (Mem-Optimizer).
The Mem-Optimizer stands as the core of our proposed UMEM framework, designed to evolve the memory bank by extracting reusable memories from executor's experience.
Crucially, to address the instance-specific noise, we introduce the Semantic Neighborhood Modeling, which constructs clusters of semantically related queries to simulate cross-task variations, and design a Marginal Utility Reward to guide the optimization process.
By maximizing this reward via Group Relative Policy Optimization (GRPO), Mem-Optimizer performs end-to-end joint optimization. This guarantees that extracted memories are not only generalizable but also intrinsically aligned with the management policy.
Besides, we implement Online Memory Evolution, where the memory bank is dynamically updated with optimal rollouts during training, forcing the agent to learn how to utilize a continuously refining memory system.
Ultimately, the trained Mem-Optimizer significantly enhancing the cross-task generalization capability of agents.

Extensive experiments across five benchmarks demonstrate that UMEM significantly outperforms highly competitive baselines like ReMem and Memp
on single-turn reasoning tasks.
Notably, ablation studies demonstrate that optimizing memory management in isolation leads to significant performance degradation, empirically validating the necessity of jointly optimizing memory extraction and management.
Further analysis confirms that Semantic Neighborhood Modeling and the Marginal Utility Reward Function effectively empower the Mem-Optimizer to distill generalizable memories from individual experiences, rather than merely memorizing instance-specific shortcuts.
Finally, results of test-time scaling evolution prove that UMEM enables agents to achieve robust and stable self-evolution, maintaining a consistent performance gain and widening the performance gap compared to baselines as interactions proceed.
These designs ensure our proposed UMEM could effectively transform interaction experience into helpful insights, paving the way for truly self-evolving agents.

\section{Related Work}

\textbf{From Parametric Memory to Non-Parametric Memory.}
Researches on memory-augmented language models have spanned from early architectural mechanisms~\cite{weston2015memorynetworks,borgeaud2022improvinglanguagemodelsretrieving} to recent scalable lookup frameworks~\cite{lan2023copyneed,cheng2026conditionalmemoryscalablelookup}.
However, these approaches necessitate computationally fine-tuning costs.
Recently, the community has converged on a non-parametric paradigm: treating external memory bank as the agent's evolvable parameters~\cite{wei2025evomemorybenchmarkingllmagent,cai2025flexcontinuousagentevolution,cai2025trainingfreegrouprelativepolicy}. 

\textbf{Self-Evolving Memory without Optimization.}
The effectiveness of non-parametric evolution hinges on \textit{how} experiences are represented. Initial attempts, such as Synapse~\cite{zheng2024synapsetrajectoryasexemplarpromptingmemory}, retrieved raw historical trajectories. However, this approach suffers from severe noise and context window inefficiencies. To distill clearer signals, subsequent works introduced structured abstraction. For example, $Mem^p$~\cite{fang2026mempexploringagentprocedural} converts trajectories into executable programs. ReasoningBank~\cite{ouyang2025reasoningbankscalingagentselfevolving} summarizes success and failure trajectories into reusable memory entries. SimpleMem~\cite{liu2026simplememefficientlifelongmemory} applies semantic compression. 
However, the memory extraction and management policy of these methods mainly rely on prompting LLMs or hand-crafted rules, preventing the further improvement of extraction and management capability.

\textbf{Self-Evolving Memory with Optimization.}
Recent research integrates optimization, like Reinforcement Learning (RL), into self-evolving agents, branching into two distinct streams:
(1) \textbf{Optimizing Working Memory or Short-term Memory}:
Approaches such as DeepAgent~\cite{deepagent}, MemAgent~\cite{xu2025amemagenticmemoryllm} and Mem-$\alpha$~\cite{wang2025memalphalearningmemoryconstruction} employ RL to manage working memory or short-term memory~\cite{jiang2025longtermmemoryfoundation}. 
While effective for handling long-context inputs, they do not construct a evolvable memory bank, which falls outside the scope of our comparison;
(2) \textbf{Optimizing Long-term Memory}: 
This stream aims to enhance the memory management capabilities of agents, exemplified by MemRL~\cite{zhang2026memrlselfevolvingagentsruntime} and EvolveR~\cite{wu2025evolverselfevolvingllmagents}. 
Existing works exhibit a critical limitation: they predominantly optimize memory selection and management while treating memory extraction as a static process~\cite{yan2026memoryr1enhancinglargelanguage}. 
Furthermore, they lack explicit mechanisms to model generalization across future queries, often resulting in the accumulation of low-quality, instance-specific noise.
In contrast, we propose the UMEM framework to jointly optimize memory extraction and management policy, ensuring that evolved memories are generalizable and aligned with future reuse.
\section{Task Formulation of Self-Evolving Agents}
\label{sec:formulation}

Self-evolving agent can be treated as a parametric system where the executor $\mathcal{E}$ (parameters $\Theta_0$) are frozen, and the external memory bank $\mathcal{B}$ serves as the evolvable, non-differentiable parameters, consisting of a set of key--value pairs $\mathcal{B}=\{(k_i, v_i)\}_{i=1}^{|\mathcal{B}|}$, where keys correspond to queries and values store the associated memory content.
In our proposed UMEM, the self-evolving process of agents is conceptualized as analogous to a network optimization process, comprising a forward pass for inference and a backward optimization for memory evolution.

\noindent\textbf{Feedforward Pass (Memory-Augmented Execution).} 
At time $t$, given a query $q$, the agent retrieves the Top-K relevant memory entries $\mathcal{B}_t^{topk}\in \mathcal{B}_t$. Then, the frozen executor $\mathcal{E}$ performs inference conditioned on this context to generate a complete trajectory $\tau_q$ and prediction $\hat{y}_t$:
$$\tau_q, \hat{y}_q \leftarrow \mathcal{E}(q, \mathcal{B}_t^{topk}; \Theta_0)$$
Here, since $\Theta_0$ is fixed, the system's performance is strictly bounded by the quality of the retrieved memory $\mathcal{B}_t^{topk}$.

\noindent\textbf{Backward Pass (Memory Bank Update).} 
The key to the self-evolving memory is to optimize memory bank $\mathcal{B}$. 
Since $\Theta_0$ is fixed, the system's performance is strictly bounded by the quality of the memory bank $\mathcal{B}_t$.
Analogous to a backward optimization process, a Memory Optimizer model (Mem-Optimizer), parameterized by $\phi$, extract memory entries (distills insights) from the trajectory $\tau_q$, and samples a pre-defined memory management operation $opt_q \in \{\texttt{ADD}, \texttt{UPDATE}, ...\}$:
$$a_q=(\Delta_q, opt_q) \sim \pi_\phi(\cdot \mid q, \tau_q, \hat{y_q})$$
where $\Delta_q$ is the extracted memory and $a_q$ represents the action to the memory bank.
The memory bank evolves after applying the action: $\mathcal{B}_{t+1} \leftarrow \text{Apply}(\mathcal{B}_t, a_q)$.
Note that while we formulate the input as the current trajectory $\tau_q$, this representation is generic; it can easily extend to extracting insights from pairs of successful or failed trajectories~\cite{ouyang2025reasoningbankscalingagentselfevolving}.

In conclusion, identifying the Mem-Optimizer ($\pi_\phi$) as the core bottleneck~\cite{zhang2026memrlselfevolvingagentsruntime,fang2026mempexploringagentprocedural,cai2025flexcontinuousagentevolution}, we propose the UMEM framework to jointly optimize its extraction and management policies.

\section{Method}
\begin{figure*}[th]  
    \centering
    \includegraphics[width=\textwidth]{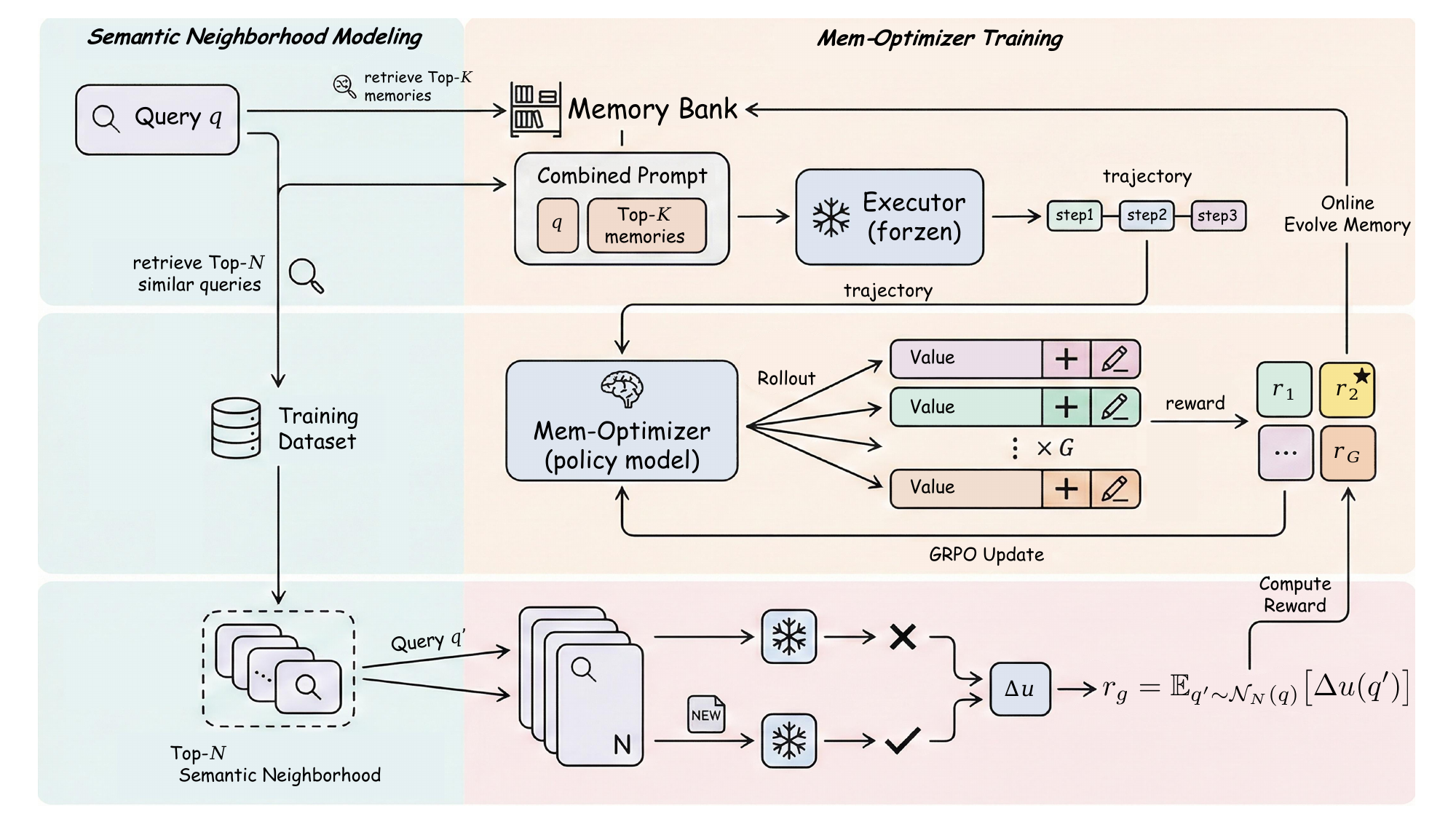}
    \caption{Overview of UMEM. \textbf{Left}: Semantic Neighborhood Modeling retrieves related queries to simulate cross-task variations. \textbf{Right}: The Mem-Optimizer distills trajectories from the frozen Executor into memory updates, which are optimized via GRPO. The process is guided by a Marginal Utility Reward that measures performance gains across the entire neighborhood to ensure generalization.}
    \label{fig:method}
\end{figure*}

This section describes our proposed UMEM framework.
To ensure generalization, we first propose Semantic Neighborhood Modeling (Section~\ref{sec:neighborhood_modeling}), which constructs query clusters to prevent overfitting. Besides, we design the Marginal Utility Reward and apply GRPO algorithm to enforce cross-task generalization (Section~\ref{sec:mem_optimizer_training}).

\subsection{Semantic Neighborhood Modeling}
\label{sec:neighborhood_modeling}

A critical risk in memory evolution is \textit{overfitting}: an extracted insight may perfectly resolve the current query but fail to generalize to related queries due to instance-specific noise or shortcuts~\cite{qin2024o1replicationjourneystrategic}.
To mitigate this, we introduce Semantic Neighborhood Modeling.
Our core insight is to treat the local cluster of similar queries as a proxy to approximate cross-task variations.
Specifically, we first project all queries into a shared semantic space using a pre-trained encoder (e.g., BGE-M3~\cite{chen-etal-2024-m3}).
For a given source query $q$, we construct its semantic neighborhood $\mathcal{N}_N(q)$ by retrieving the Top-$N$ nearest neighbors from the corpus $\mathcal{D}$ based on cosine similarity.
During training, we evaluate candidate memory updates not on the current $q$, but over the entire neighborhood $\mathcal{N}_N(q)$.
This mechanism forces the Mem-Optimizer to discard instance-specific details and extract generalizable insights.

\subsection{Mem-Optimizer Training via GRPO\label{sec:mem_optimizer_training}}

The training process of Mem-Optimizer comprising following stages: (1) Memory-Augmented Execution; (2) Mem-Optimizer Policy Rollout; (3) Marginal Utility Reward; (4) Optimization via GRPO; and (5) Online Memory Evolution. The detailed procedural flow are provided in Appendix \ref{app:mem_mgmt}.

\noindent\textbf{Memory-Augmented Execution.}
As described in Section~\ref{sec:formulation}, for each query in training dataset $q\in \mathcal{Q}$ at training step $t$, we retrieve the Top-$K$ relevant memory entries $\mathcal{B}_t^{topk}$ from the current memory bank.
The frozen executor then generates a trajectory $\tau_q$ and prediction $\hat{y_q}$. 

\noindent\textbf{Mem-Optimizer Policy Rollout.}
As shown in middle of the right panel of Figure~\ref{fig:method}, the Mem-Optimizer distills the $\tau_q$ into structured memory action $(\Delta_q, opt_q)$ (values).
Adopting the GRPO algorithm~\cite{shao2024deepseekmath}, we sample a group of $G$ memory update actions $\{a_q^{(g)}\}_{g=1}^G$:
\begin{equation}
    \{a_q^{(g)}\}_{g=1}^G \sim \pi_{\phi}(\cdot \mid q, \tau_q, \mathcal{B}_t^{topk})
\end{equation}

\noindent\textbf{Marginal Utility Reward.}
To evaluate the quality of the generated memory update actions $\{a_q^{(g)}\}_{g=1}^G$, we strictly prohibit overfitting to the single source query $q$.
Instead, we validate the memory update against the Semantic Neighborhood $\mathcal{N}_N(q)$.
For each neighbor query $q' \in \mathcal{N}_N(q)$, we compute the per-neighbor utility $\Delta u(q')$ by comparing two execution states:
a reference execution without $a_q^{(g)}$ and a memory-augmented execution where $a_q^{(g)}$ is used.
The marginal utility is then computed by sum of two terms:
(1) \textbf{Success Gain ($\mathcal{G}_{\text{succ}}$):}  It quantifies the correction of execution failures:
\begin{equation}
    \mathcal{G}_{\text{succ}}(q') = c(\tau_{q'}^{\text{mem}}) - c(\tau_{q'}^{\text{ref}})
\end{equation}
where subscripts $c(\tau_{q'}^{\text{mem}}),c(\tau_{q'}^{\text{ref}})\in\{0,1\}$ denote the correctness of the augmented and reference execution trajectories, respectively.
A positive $\mathcal{G}_{\text{succ}}$ indicates that the memory successfully fixed a previously incorrect query, while a negative value penalizes memory that introduces errors into originally correct reasoning;
(2) \textbf{Efficiency Regularization ($\mathcal{R}_{\text{eff}}$):}
Beyond correctness, high-quality memory should facilitate more efficient inference, pruning redundant and wrong reasoning steps~\cite{ahmed2025retrievalofthoughtefficientreasoningreusing,didolkar2025metacognitivereuseturningrecurring}.
To encourage concise reasoning, we introduce an Efficiency Regularization term that rewards token reduction, but \textbf{only} when correctness is preserved:
\begin{equation}
    \mathcal{R}_{\text{eff}}(q') = (c_{\text{mem}} \cdot c_{\text{ref}}) \cdot \left( 1 - \frac{\ell_{\text{mem}}^{(g)}}{\ell_{\text{ref}}} \right)
\end{equation}
Here, $\ell$ represents the length of the generated trajectory.
The gating term $(c_{\text{mem}} \cdot c_{\text{ref}})$ ensures that we do not reward brevity if it comes at the cost of accuracy (e.g., generating a short but wrong answer).
The marginal utility reward is then defined as the sum of these two reward scores:
\begin{equation}
\Delta u(q') = \mathcal{G}_{\text{succ}}(q') + \mathcal{R}_{\text{eff}}(q').
\end{equation}
The final marginal utility reward for a candidate memory is the average marginal utility over the neighborhood,
\[
r_g = \mathbb{E}_{q' \sim \mathcal{N}_N(q)} \big[ \Delta u(q') \big],
\]
which favors memory updates that both correct errors of semantically related queries.

\noindent\textbf{Optimization via GRPO.}
Finally, we train $\pi_{\phi}$ to maximize a joint objective $r_{\text{final}} = r_{\text{fmt}} + r_g$, where $r_{\text{fmt}} \in \{0, 1\}$ is a binary format reward that validates if the output format of extracted memories and management operations strictly adheres to the XML schema defined in Appendix~\ref{app:memory_template}.

\noindent\textbf{Online Memory Evolution.}
After GRPO optimization of one query $q$, we identify the memory update action $a_q^{(g)}$ with the highest marginal utility reward  and immediately apply it to the memory bank: $\mathcal{B}_{t+1} \leftarrow \text{Apply}(\mathcal{B}_t, a_q)$. 
This mechanism ensures that the memory bank is dynamically refined throughout the training process, forcing the agent to learn how to utilize and manage an evolving memory rather than a static one.

\section{Experiments}

\subsection{Setup} 
\noindent\textbf{Datasets.}
We derive our training data from the MMLU dataset \cite{hendrycks2021measuringmassivemultitasklanguage}. Specifically, we randomly sample $\sim$2,000 queries from the training split. For each query $q$, a semantic neighborhood cluster $\mathcal{N}_N(q)$ is constructed by retrieving the Top-$N$ ($N=3$) most similar samples within the training set.

\noindent\textbf{Backbone.} We employ Llama-3.2-1B-Instruct and Qwen3-4B-Instruct as the Mem-Optimizer policy $\pi_\phi$. During training (Details are in Appendix \ref{appendix:implementation_details}), Qwen3-8B serves as the frozen executor $\mathcal{E}$ to generate trajectories. To evaluate cross-model portability, we deploy the Mem-Optimizer to curate memory for diverse unseen executors, including GPT-5.1, Qwen3-8B, and Gemini-2.5-Flash. This setup assesses whether UMEM distills architectural-agnostic insights that generalize to heterogeneous and stronger models.

\begin{table*}[t]
\centering
\caption{Main Results. We evaluate UMEM using three distinct frozen executors: Qwen3-8B-Thinking, GPT-5.1, and Gemini-2.5-Flash. Performance \textcolor{BrickRed}{gains ($\uparrow$)} and \textcolor{NavyBlue}{drops ($\downarrow$)} of UMEM compared to its direct backbone are explicitly marked.}
\resizebox{\textwidth}{!}{
\begin{tabular}{lcccccccc}
\toprule[1.2pt]
\multirow{2.5}{*}{\textbf{Models}} & \multirow{2.5}{*}{\textbf{Parameters}} & \multirow{2.5}{*}{\textbf{AIME}} & \multirow{2.5}{*}{\textbf{GPQA}} & \multirow{2.5}{*}{\textbf{HLE}} & \multirow{2.5}{*}{\textbf{HotpotQA}} & \multicolumn{2}{c}{\textbf{ALFWorld}} & \multirow{2.5}{*}{\textbf{Average}} \\
\cmidrule(lr){7-8}
 & &  &  &  &  & SR & PR &  \\
\midrule
\multicolumn{9}{l}{\textit{Frozen Executor: Qwen3-8B-Thinking}} \\
\midrule
No Memory & - & \cellcolor{lightblue3}{51.67} & \cellcolor{lightblue4}{52.53} & \cellcolor{lightblue3}{7.51} & \cellcolor{lightblue6}{62.00} & \cellcolor{lightblue3}{41.04} & \cellcolor{lightblue3}{68.91} & \cellcolor{lightblue3}{47.28} \\
Self-RAG & 8B & \cellcolor{lightblue1}{30.00} & \cellcolor{lightblue1}{43.94} & \cellcolor{lightblue5}{8.56} & \cellcolor{lightblue2}{25.00} & \cellcolor{lightblue1}{30.60} & \cellcolor{lightblue1}{54.23} & \cellcolor{lightblue1}{32.06} \\
ReMem & 8B & \cellcolor{lightblue7}{\textbf{61.67}} & \cellcolor{lightblue5}{53.54} & \cellcolor{lightblue1}{6.42} & \cellcolor{lightblue1}{15.00} & \cellcolor{lightblue5}{46.27} & \cellcolor{lightblue2}{66.17} & \cellcolor{lightblue1}{41.51} \\
Memp & 8B & \cellcolor{lightblue2}{46.67} & \cellcolor{lightblue2}{49.49} & \cellcolor{lightblue7}{\textbf{11.22}} & \cellcolor{lightblue6}{62.00} & \cellcolor{lightblue4}{44.78} & \cellcolor{lightblue4}{69.78} & \cellcolor{lightblue4}{47.32} \\

Llama-3.2-1B-Instruct & 1B & \cellcolor{lightblue3}{51.67} & \cellcolor{lightblue3}{52.02} & \cellcolor{lightblue1}{6.42} & \cellcolor{lightblue3}{59.00} & \cellcolor{lightblue6}{47.01} & \cellcolor{lightblue7}{74.13} & \cellcolor{lightblue5}{48.38} \\

\textbf{UMEM-Llama-3.2-1B} (Ours) & 1B & \cellcolor{lightblue6}{60.00\gain{8.3}} & \cellcolor{lightblue7}{\textbf{54.04}\gain{2.0}} & \cellcolor{lightblue2}{6.95\gain{0.5}} & \cellcolor{lightblue5}{61.00\gain{2}} & \cellcolor{lightblue4}{44.78\loss{2.2}} & \cellcolor{lightblue5}{72.14\loss{2.0}} & \cellcolor{lightblue6}{49.82\gain{1.4}} \\

Qwen3-4B-Instruct & 4B & \cellcolor{lightblue4}{55.00} & \cellcolor{lightblue5}{53.54} & \cellcolor{lightblue2}{6.95} & \cellcolor{lightblue4}{60.00} & \cellcolor{lightblue2}{40.30} & \cellcolor{lightblue1}{65.92} & \cellcolor{lightblue2}{46.95} \\

\textbf{UMEM-Qwen3-4B} (Ours) & 4B & \cellcolor{lightblue5}{58.33\gain{3.3}} & \cellcolor{lightblue3}{52.02\loss{1.5}} & \cellcolor{lightblue4}{8.02\gain{1.1}} & \cellcolor{lightblue7}{\textbf{63.00}\gain{3}} & \cellcolor{lightblue7}{\textbf{50.75}\gain{10.5}} & \cellcolor{lightblue6}{\textbf{73.13}\gain{7.2}} & \cellcolor{lightblue7}{\textbf{50.88}\gain{3.9}} \\
\midrule
\multicolumn{9}{l}{\textit{Frozen Executor: GPT-5.1}} \\
\midrule
No Memory & - & \cellcolor{lightblue2}{40.00} & \cellcolor{lightblue1}{57.57} & \cellcolor{lightblue1}{6.95} & \cellcolor{lightblue1}{39.00} & \cellcolor{lightblue1}{61.94} & \cellcolor{lightblue1}{66.67} & \cellcolor{lightblue1}{45.36} \\
Self-RAG & \texttt{API} & \cellcolor{lightblue6}{50.00} & \cellcolor{lightblue2}{57.58} & \cellcolor{lightblue2}{7.49} & \cellcolor{lightblue2}{42.00} & \cellcolor{lightblue3}{70.90} & \cellcolor{lightblue6}{83.71} & \cellcolor{lightblue4}{51.95} \\
ReMem & \texttt{API} & \cellcolor{lightblue1}{30.00} & \cellcolor{lightblue5}{62.63} & \cellcolor{lightblue5}{8.56} & \cellcolor{lightblue3}{43.00} & \cellcolor{lightblue4}{73.13} & \cellcolor{lightblue4}{79.60} & \cellcolor{lightblue1}{49.49} \\
Memp & \texttt{API} & \cellcolor{lightblue4}{45.00} & \cellcolor{lightblue4}{62.12} & \cellcolor{lightblue7}{\textbf{10.16}} & \cellcolor{lightblue5}{52.00} & \cellcolor{lightblue5}{77.61} & \cellcolor{lightblue5}{81.34} & \cellcolor{lightblue6}{54.71} \\
Llama-3.2-1B-Instruct & 1B & \cellcolor{lightblue3}{43.33} & \cellcolor{lightblue3}{61.11} & \cellcolor{lightblue2}{7.49} & \cellcolor{lightblue4}{51.00} & \cellcolor{lightblue1}{61.94} & \cellcolor{lightblue1}{73.63} & \cellcolor{lightblue2}{49.75} \\
\textbf{UMEM-Llama-3.2-1B} (Ours) & 1B & \cellcolor{lightblue4}{45.00\gain{1.7}} & \cellcolor{lightblue5}{62.63\gain{1.5}} & \cellcolor{lightblue5}{8.56\gain{1.1}} & \cellcolor{lightblue7}{\textbf{55.00}\gain{4}} & \cellcolor{lightblue2}{64.18\gain{2.2}} & \cellcolor{lightblue2}{75.37\gain{1.7}} & \cellcolor{lightblue3}{51.79\gain{2.0}} \\
Qwen3-4B-Instruct & 4B & \cellcolor{lightblue5}{46.67} & \cellcolor{lightblue5}{62.63} & \cellcolor{lightblue4}{8.02} & \cellcolor{lightblue5}{52.00} & \cellcolor{lightblue3}{70.90} & \cellcolor{lightblue3}{78.86} & \cellcolor{lightblue5}{53.18} \\
\textbf{UMEM-Qwen3-4B} (Ours) & 4B & \cellcolor{lightblue7}{\textbf{51.67}\gain{5.0}} & \cellcolor{lightblue7}{\textbf{65.15}\gain{2.5}} & \cellcolor{lightblue5}{8.56\gain{0.5}} & \cellcolor{lightblue6}{54.00\gain{2.0}} & \cellcolor{lightblue7}{\textbf{82.84}\gain{11.9}} & \cellcolor{lightblue7}{\textbf{84.20}\gain{5.3}} & \cellcolor{lightblue7}{\textbf{57.74}\gain{4.6}} \\
\midrule
\multicolumn{9}{l}{\textit{Frozen Executor: Gemini-2.5-Flash}} \\
\midrule
No Memory & - & \cellcolor{lightblue2}{53.33} & \cellcolor{lightblue4}{73.23} & \cellcolor{lightblue3}{10.16} & \cellcolor{lightblue1}{30.00} & \cellcolor{lightblue2}{55.22} & \cellcolor{lightblue2}{72.89} & \cellcolor{lightblue1}{49.14} \\
Self-RAG & \texttt{API} & \cellcolor{lightblue4}{56.67} & \cellcolor{lightblue2}{71.72} & \cellcolor{lightblue3}{10.16} & \cellcolor{lightblue5}{42.00} & \cellcolor{lightblue5}{59.70} & \cellcolor{lightblue3}{74.50} & \cellcolor{lightblue5}{52.46} \\
ReMem & \texttt{API} & \cellcolor{lightblue4}{56.67} & \cellcolor{lightblue1}{70.20} & \cellcolor{lightblue4}{10.70} & \cellcolor{lightblue2}{36.00} & \cellcolor{lightblue3}{56.72} & \cellcolor{lightblue4}{75.62} & \cellcolor{lightblue3}{50.99} \\
Memp & \texttt{API} & \cellcolor{lightblue2}{53.33} & \cellcolor{lightblue5}{74.75} & \cellcolor{lightblue1}{7.49} & \cellcolor{lightblue4}{41.00} & \cellcolor{lightblue6}{60.45} & \cellcolor{lightblue5}{76.74} & \cellcolor{lightblue4}{52.29} \\
Llama-3.2-1B-Instruct & 1B & \cellcolor{lightblue1}{51.67} & \cellcolor{lightblue4}{73.23} & \cellcolor{lightblue3}{10.16} & \cellcolor{lightblue5}{42.00} & \cellcolor{lightblue1}{53.73} & \cellcolor{lightblue1}{71.27} & \cellcolor{lightblue1}{50.34} \\
\textbf{UMEM-Llama-3.2-1B} (Ours) & 1B & \cellcolor{lightblue5}{58.33\gain{6.7}} & \cellcolor{lightblue2}{71.72\loss{1.5}} & \cellcolor{lightblue7}{\textbf{13.37}\gain{3.2}} & \cellcolor{lightblue5}{42.00} & \cellcolor{lightblue4}{58.96\gain{5.2}} & \cellcolor{lightblue6}{77.24\gain{6.0}} & \cellcolor{lightblue6}{53.60\gain{3.3}} \\
Qwen3-4B-Instruct & 4B & \cellcolor{lightblue4}{56.67} & \cellcolor{lightblue3}{72.22} & \cellcolor{lightblue2}{9.63} & \cellcolor{lightblue5}{42.00} & \cellcolor{lightblue1}{52.24} & \cellcolor{lightblue1}{72.64} & \cellcolor{lightblue2}{50.90} \\
\textbf{UMEM-Qwen3-4B} (Ours) & 4B & \cellcolor{lightblue7}{\textbf{60.00}\gain{3.3}} & \cellcolor{lightblue7}{\textbf{76.26}\gain{4.0}} & \cellcolor{lightblue5}{11.76\gain{2.1}} & \cellcolor{lightblue7}{\textbf{45.00}\gain{3.0}} & \cellcolor{lightblue7}{\textbf{61.19}\gain{9.0}} & \cellcolor{lightblue7}{\textbf{78.61}\gain{6.0}} & \cellcolor{lightblue7}{\textbf{55.47}\gain{4.6}} \\
\bottomrule[1.2pt]
\end{tabular}%
}
\label{tab:main_results}
\end{table*}

\noindent\textbf{Baselines.} 
We evaluate UMEM against several representative paradigms: (1) No Memory, which assesses the frozen backbone LLM without external memory; (2) No Train, a non-learning ablation using identical prompt templates without policy training; (3) Self-RAG~\cite{asai2024selfrag}, which filters retrieved context via inference-time self-critique; (4) Memp~\cite{fang2026mempexploringagentprocedural}, a decoupled pipeline baseline that distills trajectories into fine-grained instructions and high-level scripts through independent Build-Retrieve-Update stages; and (5) ReMem~\cite{wei2025evomemorybenchmarkingllmagent}, a baseline focusing on memory management that maintains trajectory-level memory via discrete operations interleaved with reasoning steps. Unlike these methods, UMEM uniquely targets the joint optimization and granularity alignment of memory extraction and management.

\noindent\textbf{Benchmark.}
We evaluate UMEM on five benchmarks designed to assess memory stability and reusability across single-turn reasoning and multi-turn embodied interaction. For single-turn tasks, we select AIME (merging AIME24 and AIME25) ~\cite{aime24,aime25} and GPQA-Diamond~\cite{rein2023gpqa} to test domain-specific mathematical and scientific reasoning, alongside HLE~\cite{phan2025hle} for multidisciplinary complex reasoning. We also include HotpotQA~\cite{yang2018hotpotqa} to evaluate strategy reuse in multi-hop question answering. For these single-turn benchmarks, performance is reported using Exact Match (EM) accuracy. For multi-turn embodied settings, we adopt ALFWorld~\cite{shridhar2021alfworld}, which requires long-horizon planning and state-dependent decision-making; we report Cumulative Success Rate (CSR) and Progress Rate following prior benchmark/metric practice~\cite{wu2024streambenchbenchmarkingcontinuousimprovement,wei2025evomemorybenchmarkingllmagent}. 

\noindent\textbf{Evaluation Protocol.} We adopt a streaming protocol to assess the agent's continuous self-evolution. Unlike static benchmarks, tasks are processed as a sequential stream.
This zero-reset setup ensures that experiences distilled from trajectory are immediately integrated into memory bank to facilitate the reasoning of all subsequent queries.


\subsection{Main Results}
As illustrated in Table \ref{tab:main_results}, UMEM consistently outperforms all baseline methods, including state-of-the-art memory management systems like ReMem and Memp, across the vast majority of benchmarks. Notably, our framework achieves significant performance leaps in complex reasoning tasks (e.g., AIME and GPQA\_Diamond) and embodied environments like ALFWorld, where UMEM-Qwen3-4B attains a Success Rate of 82.84\% when paired with GPT-5.1. 

A key observation is that the effectiveness of UMEM is positively correlated with the strength of the frozen executor; more powerful executors such as GPT-5.1 and Gemini-2.5-Flash tend to yield more pronounced gains compared to the Qwen3-8B-Thinking baseline. This phenomenon can be attributed to the higher-quality reasoning trajectories and interaction traces produced by stronger executors, which serve as high-fidelity source material for UMEM to distill more actionable and sophisticated insights. 

Furthermore, UMEM exhibits excellent scalability regarding its policy model size. While even a compact 1B model (UMEM-Llama-3.2-1B) provides a substantial improvement over the base model and often surpasses larger models, further scaling the policy model to 4B consistently yields additional performance dividends across nearly all tasks. This suggests that while UMEM is highly efficient at a small scale, increased model capacity allows it to capture more nuanced semantic relationships and implement more precise memory management strategies, thereby further pushing the performance upper bound of self-evolving agents.

\subsection{Ablation Studies}

\begin{table*}[t]
\centering
\caption{Ablation studies on joint optimization components and neighborhood size on \textbf{GPT-5.1} and \textbf{Qwen3-8B-Thinking}. The full UMEM method for each model serves as the baseline. The performance drops (\loss{drop}) or gains (\gain{gain}) of each variant compared to the respective full method are explicitly marked to demonstrate the contribution of each component. Opt. denotes Optimization. SNM denotes the Semantic Neighborhood Modeling.}
\label{tab:ablation}
\resizebox{\textwidth}{!}{%

\begin{tabular}{lcccc cccccc cc}
\toprule[1.2pt]
\multirow{3.5}{*}{\textbf{Method}} & \multicolumn{6}{c}{\textbf{GPT-5.1}} & \multicolumn{6}{c}{\textbf{Qwen3-8B-Thinking}} \\
\cmidrule[0.6pt](lr){2-7} \cmidrule[0.6pt](lr){8-13}
 & \textbf{AIME} & \textbf{GPQA} & \textbf{HLE} & \textbf{HotpotQA} & \multicolumn{2}{c}{\textbf{ALFWorld}} & \textbf{AIME} & \textbf{GPQA} & \textbf{HLE} & \textbf{HotpotQA} & \multicolumn{2}{c}{\textbf{ALFWorld}} \\
\cmidrule[0.6pt](lr){6-7} \cmidrule[0.6pt](lr){12-13}
 & (Acc.) & (Acc.) & (Acc.) & (Acc.) & SR & PR & (Acc.) & (Acc.) & (Acc.) & (Acc.) & SR & PR \\
 
\cmidrule[0.6pt](lr){1-13}
\multicolumn{13}{l}{\textit{Joint Optimization Components}} \\ 
\cmidrule[0.6pt](lr){1-13}

\textbf{UMEM (Full Method)} & \cellcolor{lightblue7}{\textbf{51.67}} & \cellcolor{lightblue7}{\textbf{65.15}} & \cellcolor{lightblue5}{8.56} & \cellcolor{lightblue5}{54.00} & \cellcolor{lightblue7}{\textbf{82.84}} & \cellcolor{lightblue5}{84.20} & \cellcolor{lightblue7}{\textbf{58.33}} & \cellcolor{lightblue7}{\textbf{53.54}} & \cellcolor{lightblue5}{8.02} & \cellcolor{lightblue7}{\textbf{63.00}} & \cellcolor{lightblue5}{50.75} & \cellcolor{lightblue5}{\textbf{73.13}} \\
\hspace{3mm} w/o Extraction Opt. & \cellcolor{lightblue2}{45.00\loss{6.7}} & \cellcolor{lightblue1}{59.60\loss{5.6}} & \cellcolor{lightblue1}{5.88\loss{2.7}} & \cellcolor{lightblue2}{51.00\loss{3.0}} & \cellcolor{lightblue1}{76.12\loss{6.7}} & \cellcolor{lightblue1}{80.72\loss{3.5}} & \cellcolor{lightblue3}{55.00\loss{3.3}} & \cellcolor{lightblue2}{51.01\loss{2.5}} & \cellcolor{lightblue5}{8.02\gain{0.0}} & \cellcolor{lightblue4}{61.00\loss{2.0}} & \cellcolor{lightblue4}{45.53\loss{5.2}} & \cellcolor{lightblue2}{67.79\loss{5.3}} \\
\hspace{3mm} w/o Management Opt. & \cellcolor{lightblue5}{48.33\loss{3.3}} & \cellcolor{lightblue5}{64.65\loss{0.5}} & \cellcolor{lightblue7}{\textbf{9.09}\gain{0.5}} & \cellcolor{lightblue7}{\textbf{55.00}\gain{1.0}} & \cellcolor{lightblue4}{80.60\loss{2.2}} & \cellcolor{lightblue7}{\textbf{84.33}\gain{0.1}} & \cellcolor{lightblue5}{56.67\loss{1.7}} & \cellcolor{lightblue5}{53.03\loss{0.5}} & \cellcolor{lightblue1}{6.95\loss{1.1}} & \cellcolor{lightblue7}{63.00\gain{0.0}} & \cellcolor{lightblue3}{44.70\loss{6.1}} & \cellcolor{lightblue3}{69.03\loss{4.1}} \\
\hspace{3mm} w/o SNM & \cellcolor{lightblue1}{41.67\loss{10.0}} & \cellcolor{lightblue4}{64.14\loss{1.0}} & \cellcolor{lightblue2}{6.95\loss{1.6}} & \cellcolor{lightblue4}{52.00\loss{2.0}} & \cellcolor{lightblue3}{79.10\loss{3.7}} & \cellcolor{lightblue2}{81.09\loss{3.1}} & \cellcolor{lightblue3}{55.00\loss{3.3}} & \cellcolor{lightblue1}{50.00\loss{3.5}} & \cellcolor{lightblue3}{7.49\loss{0.5}} & \cellcolor{lightblue2}{60.00\loss{3.0}} & \cellcolor{lightblue7}{\textbf{52.99}\gain{2.2}} & \cellcolor{lightblue4}{72.30\loss{0.8}} \\

\cmidrule[0.6pt](lr){1-13}
\multicolumn{13}{l}{\textit{Sensitivity to Neighborhood Size}} \\ 
\cmidrule[0.6pt](lr){1-13}

\textbf{UMEM ($N=3$, Ours)} & \cellcolor{lightblue7}{\textbf{51.67}} & \cellcolor{lightblue7}{\textbf{65.15}} & \cellcolor{lightblue5}{\textbf{8.56}} & \cellcolor{lightblue5}{\textbf{54.00}} & \cellcolor{lightblue7}{\textbf{82.84}} & \cellcolor{lightblue5}{\textbf{84.20}} & \cellcolor{lightblue7}{\textbf{58.33}} & \cellcolor{lightblue7}{\textbf{53.54}} & \cellcolor{lightblue5}{8.02} & \cellcolor{lightblue7}{\textbf{63.00}} & \cellcolor{lightblue5}{\textbf{50.75}} & \cellcolor{lightblue5}{73.13} \\
\hspace{3mm} $N=1$ (Too Narrow) & \cellcolor{lightblue5}{48.33\loss{3.3}} & \cellcolor{lightblue2}{63.13\loss{2.0}} & \cellcolor{lightblue4}{7.49\loss{1.1}} & \cellcolor{lightblue1}{50.00\loss{4.0}} & \cellcolor{lightblue2}{78.36\loss{4.5}} & \cellcolor{lightblue3}{81.34\loss{2.9}} & \cellcolor{lightblue1}{51.67\loss{6.7}} & \cellcolor{lightblue3}{51.52\loss{2.0}} & \cellcolor{lightblue5}{8.02\gain{0.0}} & \cellcolor{lightblue1}{59.00\loss{4.0}} & \cellcolor{lightblue1}{41.79\loss{9.0}} & \cellcolor{lightblue1}{67.66\loss{5.5}} \\
\hspace{3mm} $N=5$ (Too Broad) & \cellcolor{lightblue4}{46.67\loss{5.0}} & \cellcolor{lightblue5}{64.65\loss{0.5}} & \cellcolor{lightblue4}{7.49\loss{1.1}} & \cellcolor{lightblue4}{52.00\loss{2.0}} & \cellcolor{lightblue5}{81.34\loss{1.5}} & \cellcolor{lightblue4}{84.08\loss{0.1}} & \cellcolor{lightblue7}{58.33\gain{0.0}} & \cellcolor{lightblue4}{52.53\loss{1.0}} & \cellcolor{lightblue7}{\textbf{9.63}\gain{1.6}} & \cellcolor{lightblue5}{62.00\loss{1.0}} & \cellcolor{lightblue2}{42.54\loss{8.2}} & \cellcolor{lightblue7}{\textbf{73.26}\gain{0.1}} \\
\bottomrule[1.2pt]
\end{tabular}%
}
\end{table*}

This section validates the effectiveness of our designs in UMEM by ablation studies: 
(1) the necessity and sensitivity of semantic neighborhood modeling; and (2) joint optimization on memory extraction and management.

\noindent\textbf{Semantic Neighborhood Modeling.}
We first examine the necessity of Semantic Neighborhood Modeling.
The forth row in Table~\ref{tab:ablation} reveals that removing it during training results in significant performance collapse, particularly on the reasoning-heavy AIME benchmark (GPT-5.1: dropping from 51.67 to 41.67; Qwen3-8B: dropping from 58.33 to 55.00).
Furthermore, we also investigate the impact of the semantic neighborhood size $N \in \{1, 3, 5\}$. As reported in last three rows in Table~\ref{tab:ablation}, $N=3$ yields the optimal balance between task-specific optimization and cross-task transfer. 
Performance degrades at both extremes: an overly narrow neighborhood ($N=1$) fails to capture task shifts (GPT-5.1: AIME drops to 48.33; Qwen3-8B: dropping from 58.33 to 51.67), while an overly broad one ($N=5$) introduces noise that dilutes the reward signal during optimization.

\noindent\textbf{Joint Optimization.}
We evaluate the contribution of memory extraction and  management by masking the ``gradient" of their respective tokens. 
As shown in the first two rows of Table~\ref{tab:ablation}, breaking the joint optimization leads to severe performance degradation across the majority of benchmarks.
Specifically, disabling memory extraction optimization results in a average performance decline of 4.7 points across all metrics, which is significantly higher than that observed when removing management optimization (0.73 points).
These results reveal that optimizing the quality of extracted memory is the more important for effective self-evolution.



\subsection{Stability of Self-Evolution\label{subsec:evolution}}

\begin{figure*}[t]
    \centering
    \includegraphics[width=0.94\textwidth]{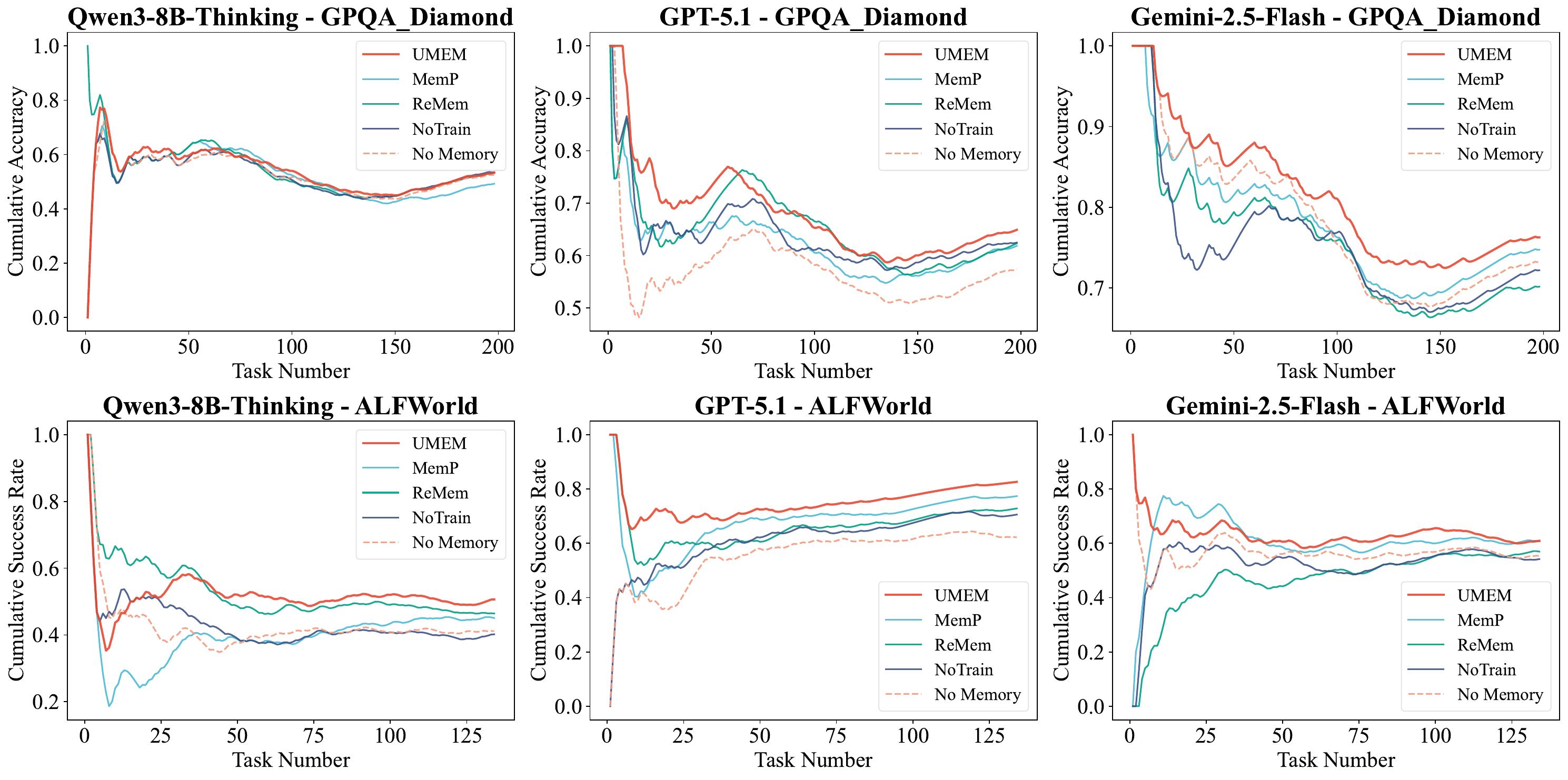}
    \caption{Cumulative performance over sequential tasks on GPQA-Diamond and ALFWorld Benchmarks.}
    \label{fig:cumulative_acc}
\end{figure*}


We evaluate UMEM under a continual learning setting across both single-turn reasoning benchmarks and the multi-turn ALFWorld environment, reporting the cumulative accuracy in Figure~\ref{fig:cumulative_acc}. 
In this streaming protocol, the agent must continuously evolve its memory bank without resetting. This poses a severe challenge: error accumulation. As interaction proceeds, flawed memory extraction policies tend to pollute the memory bank with noise or instance-specific shortcuts, degrading performance on subsequent tasks.
As shown in Figure~\ref{fig:cumulative_acc}, under this challenging setting, UMEM consistently maintains a superior performance curve compared to baselines, particularly in the later stages. It exhibits significantly slower and more controlled degradation than ReMem and MemP across all evaluations, with the performance gap widening as interaction proceeds.
Crucially, ReMem (green curve), which optimizes memory management in isolation, suffers the most rapid degradation and results in the lowest final performance, proving the necessity of jointly optimization.
This behavior indicates that UMEM accumulates fewer harmful memories over long horizons, and that its advantage stems not from whether memory is learned, but from how memory extraction and management are coordinated during continual evolution.

The extracted memories of the baselines like ReMem and Memp may appear locally better, yet their long-horizon utility remains opaque to the memory manager.
Consequently, such memories are often retained and repeatedly reused even when they introduce subtle reasoning errors, leading to progressive error amplification in cumulative evaluation.
In contrast, the substantially reduced degradation observed for UMEM suggests that newly updated memories are more consistently aligned with future reuse.

Taken together, these results support the conclusion that stable self-evolution requires memory updates to be tightly coupled with the context in which errors arise.
By evolving memory primarily around experiences most relevant to the current trajectory, UMEM promotes structured knowledge consolidation rather than unconstrained accumulation.
From an optimization perspective, this behavior corresponds to sparse, localized updates over external memory parameters, which naturally limit interference and mitigate long-horizon error accumulation.

\subsection{Test-Time Self-Evolution}

\begin{figure}[t]
    \centering
    \includegraphics[width=\linewidth]{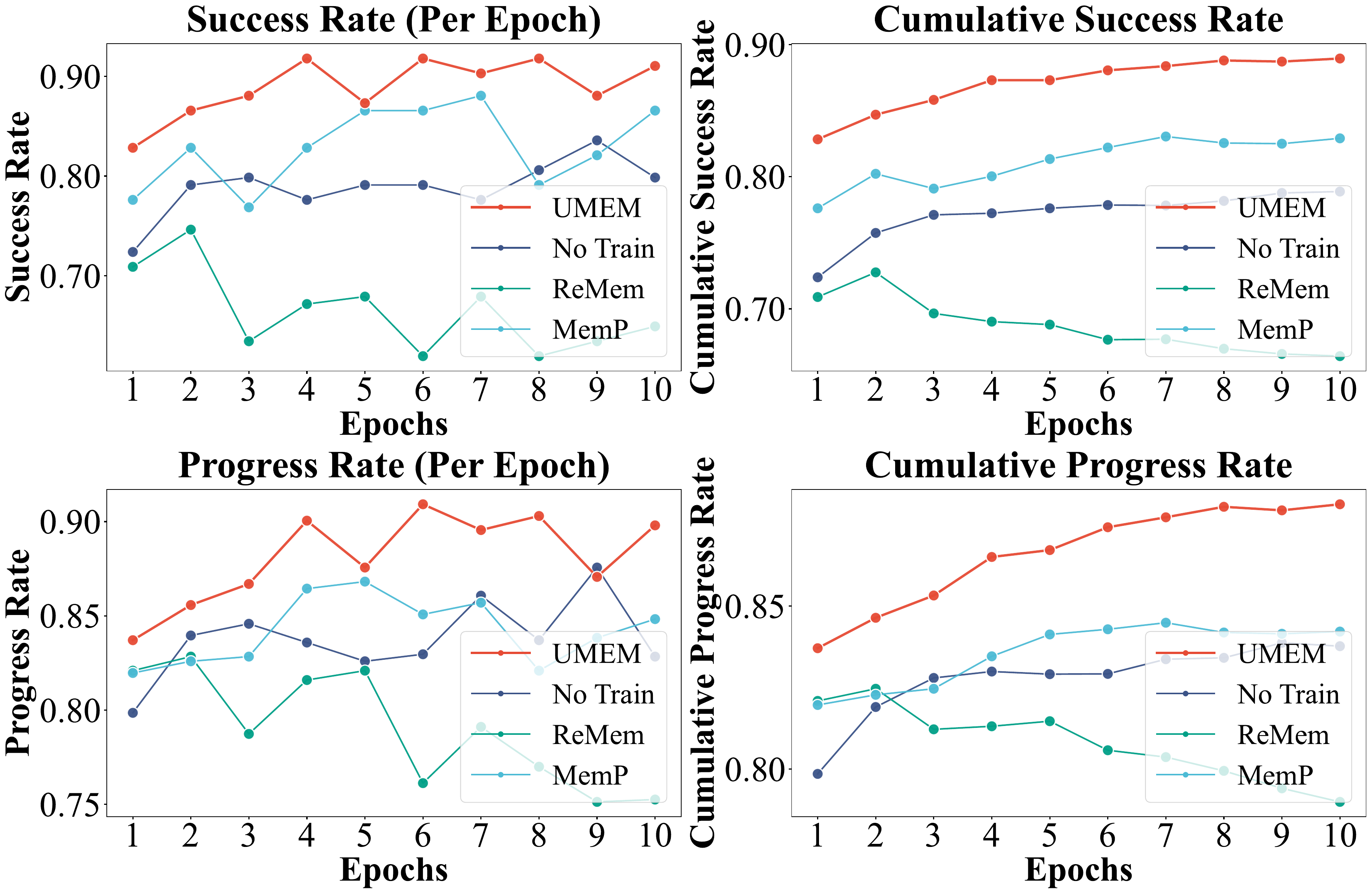}
    \caption{Test-Time Self-Evolution on ALFWorld.}
    \label{fig:test_time_self_evolve}
\end{figure}

To further validate the sustainability of self-evolution beyond the single-epoch setting in Section~\ref{subsec:evolution}, we extend the experimental scope from 1 epoch to a rigorous 10-epoch long-horizon continual interaction on the ALFWorld benchmark with GPT-5.1 as the executor.
Figure~\ref{fig:test_time_self_evolve} reports both epoch-wise and cumulative Success Rate and Progress Rate.
As shown in the per-epoch Success Rate, UMEM consistently achieves the highest performance across all epochs.
Although online retrieval and memory updates inevitably introduce performance fluctuations, UMEM recovers quickly after temporary drops, indicating a well-balanced memory strategy between exploration and stability during continual evolution.
The cumulative Success Rate further highlights UMEM’s advantage. UMEM shows a steady and sustained improvement trend, converging to a substantially higher performance level than all baselines.
Beyond final task success, UMEM also consistently outperforms baselines on Progress Rate, with a particularly pronounced margin in cumulative metrics. This trend suggests that, even in partially unsuccessful episodes, UMEM tends to execute more correct intermediate steps, reflecting more stable multi-step decision-making.
Overall, these results indicate that UMEM supports a more stable and sustainable form of agent self-evolution under continual interaction.

\subsection{Cross-Model Effectiveness and Efficiency}

\begin{figure}[t]
    \centering
    \includegraphics[width=\linewidth]{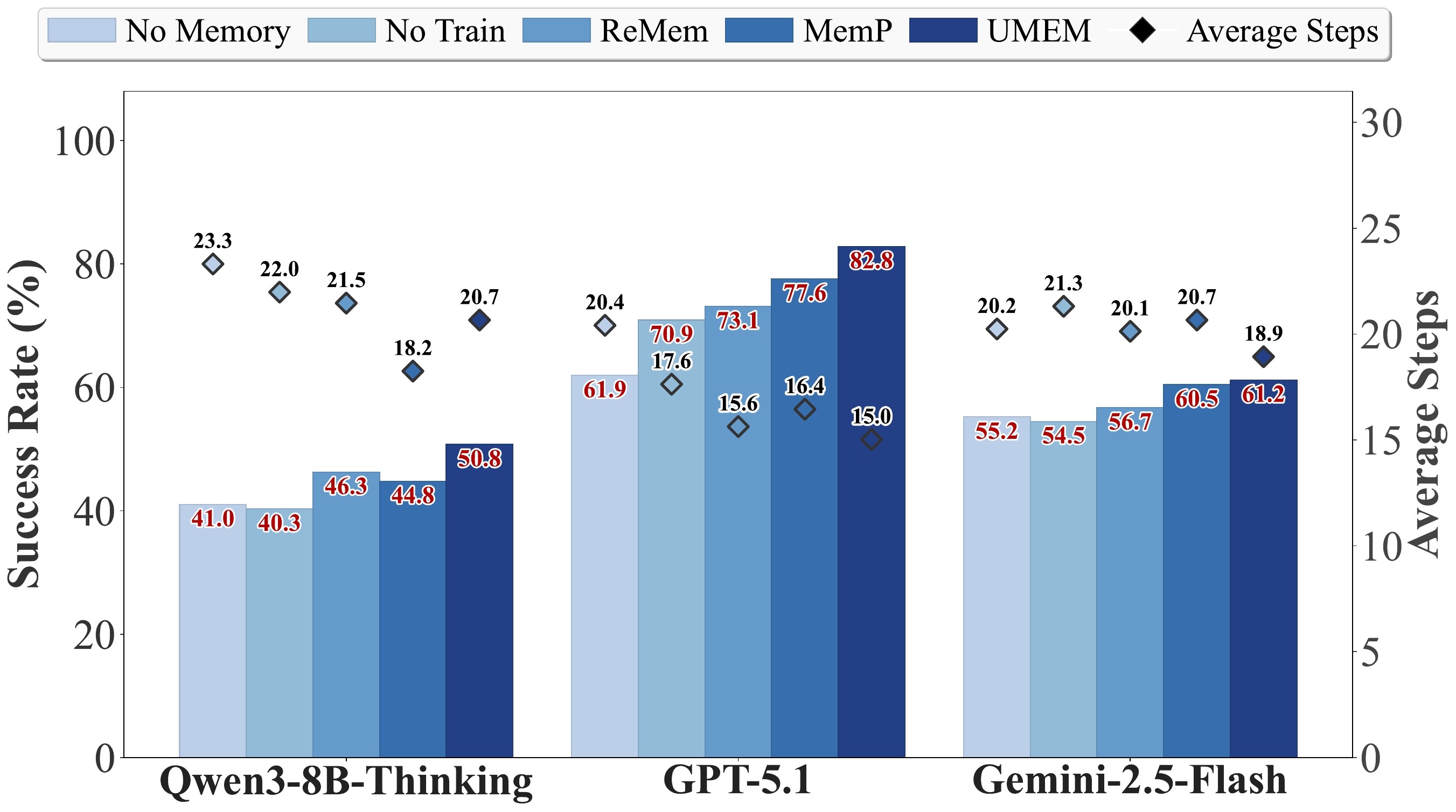}
    \caption{Success Rate and Average Steps on ALFWorld benchmark across different executor models.}
    \label{fig:step_efficiency}
\end{figure}

Figure~\ref{fig:step_efficiency} reports Success Rate and Average Steps on ALFWorld across different executor LLMs.
UMEM consistently achieves the highest Success Rate for all executors, indicating that the evolved experiences provide robust, executor-agnostic performance gains.
Notably, this improvement is accompanied by a clear reduction in Average Steps, showing that higher success is not obtained through longer or more exploratory interaction trajectories, but through more efficient decision making during interaction. This efficiency gain is evident in case study \ref{fig:case_study}.

The joint improvement in success and efficiency provides insight into the nature of the experiences evolved by UMEM.
In long-horizon interactive tasks, overly specific experiences often lead to shortcut behaviors that fail to generalize to similar tasks, ultimately causing execution failures; in contrast, overly coarse heuristics fail to sufficiently constrain execution and result in longer trajectories. 
Across all executor models, UMEM consistently avoids these failure modes, achieving higher success with fewer execution steps.
This pattern indicates that the observed gains reflect a genuine improvement in execution efficiency that generalizes across executors, rather than an artifact of increased interaction length or model-specific behavior.

\section{Conclusion}

In this paper, we introduced UMEM for self-evolving agents. Unlike prior approaches that treat memory extraction and management as static or decoupled processes, UMEM achieves joint optimization of extraction and management through Semantic Neighborhood Modeling and GRPO augmented with a Marginal Utility Reward. 
This design effectively mitigates the accumulation of instance-specific noise and ensures that extracted memories are intrinsically aligned with the agent's management policy. 
Empirical results demonstrate that UMEM significantly outperforms highly competitive baselines in both cross-task generalization and execution efficiency. 
By enabling agents to continuously refine the memory bank during continuous interaction, UMEM offers a robust paradigm for realizing lifelong learning in open-ended environments.

\section*{Impact Statement}
This paper presents work whose goal is to advance the field of Machine
Learning. There are many potential societal consequences of our work, none
of which we feel must be specifically highlighted here.

\bibliography{example_paper}
\bibliographystyle{icml2026}

\newpage
\appendix
\onecolumn
\section{Implementation Details}
\label{appendix:implementation_details}
We optimize the Mem-Optimizer using GRPO~\citep{shao2024deepseekmath}.
For each update, we sample a batch of 128 training queries and generate $G{=}8$ rollouts per query.
Training is conducted for 3 epochs.
Semantic neighborhoods are constructed with Top-$N{=}3$ neighbors, while retrieval during memory evolution uses Top-$K{=}3$ memories.
We apply KL regularization with coefficient $\beta{=}0.001$ and use a clipping ratio of $\epsilon{=}0.2$.
The learning rate is set to $1\times10^{-6}$.
During training, generation is performed with temperature 1.0 to encourage exploration.
At evaluation time, as well as for executor inference, we use greedy decoding with temperature 0.0.
Our method is implemented using the \texttt{verl} framework~\citep{sheng2024verl} and trained on 16 NVIDIA A100 GPUs for approximately 11 hours.

\section{Mem-Optimizer Action Template}
\label{app:memory_template}

Each Mem-Optimizer action is represented as a structured output following the template:

\begin{center}
\texttt{<experience><value>\ldots</value><operation>\ldots</operation></experience>}
\end{center}
where \texttt{<value>} encodes the extracted memory content derived from the interaction trace, and \texttt{<operation>} specifies the corresponding memory evolution decision (e.g., addition, replacement).

\section{Theoretical Analysis}
\subsection{Cosine Neighborhood as a Proxy for Reuse-Semantic Proximity}
\label{app:cosine_proxy}

\begin{lemma}[Retrieval-score stability under cosine proximity]
Let $e(\cdot)$ be $\ell_2$-normalized embeddings, i.e., $\|e(x)\|_2=1$.
For any two queries $q_1,q_2$ and any candidate key $k$,
\[
\big| e(q_1)^\top e(k) - e(q_2)^\top e(k)\big|
\le \|e(q_1)-e(q_2)\|_2
= \sqrt{2-2\,e(q_1)^\top e(q_2)}.
\]
\end{lemma}

\begin{proof}
Since $\|e(k)\|_2=1$, by Cauchy--Schwarz,
\[
\big| e(q_1)^\top e(k) - e(q_2)^\top e(k)\big|
= \big| (e(q_1)-e(q_2))^\top e(k)\big|
\le \|e(q_1)-e(q_2)\|_2.
\]
For unit vectors, $\|u-v\|_2^2 = 2-2u^\top v$, hence the equality.
\end{proof}

\paragraph{Interpretation.}
High cosine similarity guarantees that $q_1$ and $q_2$ assign nearly identical relevance scores to any memory key.
This score stability ensures highly overlapping retrieval rankings (and thus similar Top-$K$ sets).
Consequently, the cosine neighborhood of a source query effectively captures the cluster of future queries that will likely retrieve (and reuse) the same memory.

\section{Prompt Templates}
\label{app:prompts}

We present the detailed instruction templates used in our framework, encompassing both the Memory Optimizer and the Executor LLM.
First, the system prompt for the Memory Optimizer, which is responsible for refining and organizing retrieved past experiences, is shown in Prompt Prompt \ref{fig:full_summarizer_comparison}.
For the Executor LLM, we designed distinct system prompts to adhere to specific output formats across different domains during training and evaluation. Specifically, mathematical reasoning tasks follow the instructions in Prompt \ref{prompt:math}. The unified template for multiple-choice questions (handling both index-based and letter-based outputs) is presented in Prompt \ref{prompt:mcq}, while general question-answering tasks are guided by Prompt \ref{prompt:qa}.

\begin{figure*}[t]
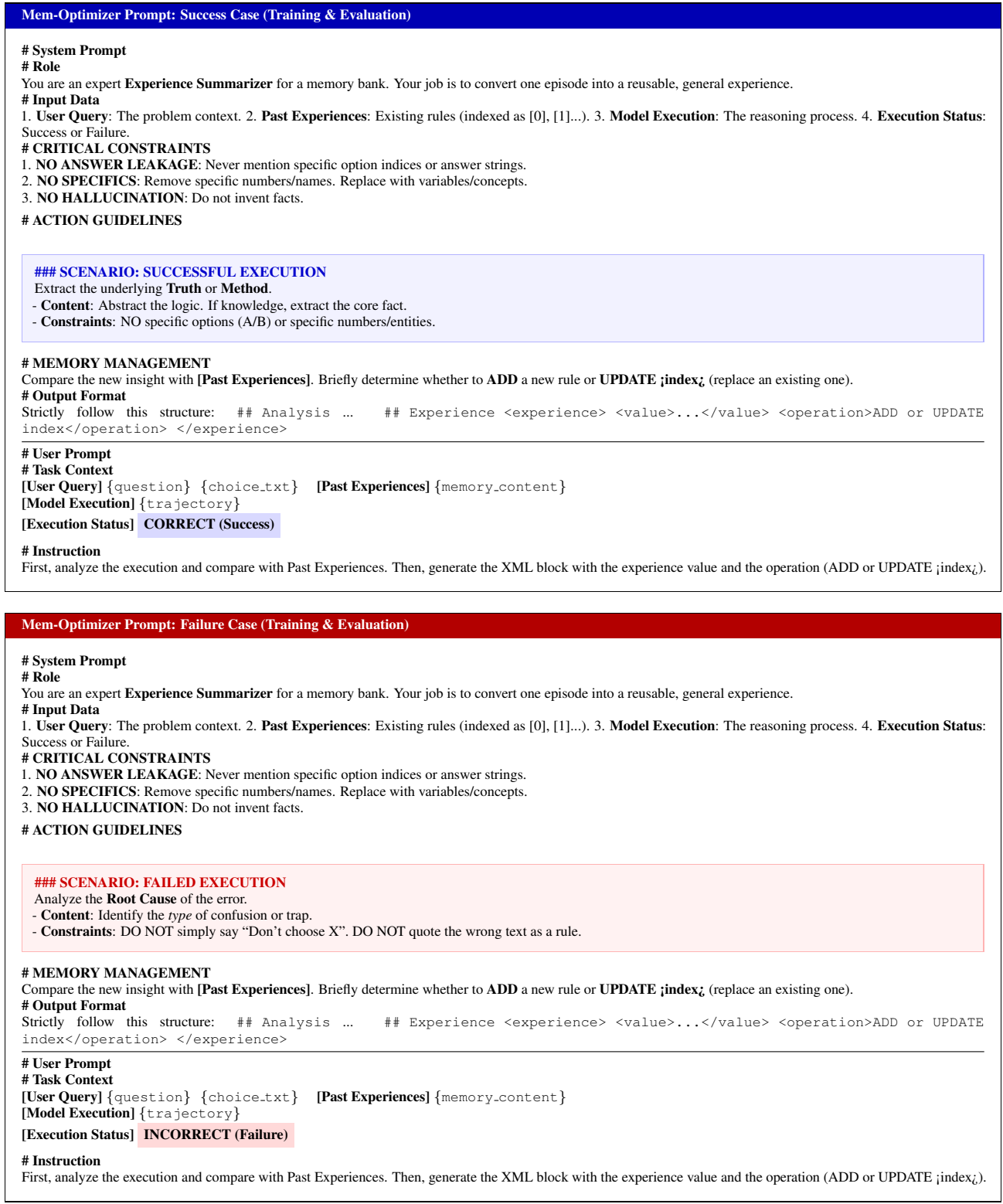

\centering
\scriptsize
\begin{tcbraster}[
    raster columns=1, 
    raster row skip=10pt, 
    sharp corners, 
    boxrule=0.5pt,
    colback=white,
    colframe=black,
    fonttitle=\bfseries,
    left=5pt, right=5pt, top=5pt, bottom=5pt
]                                                                                                                                                               
    \begin{tcolorbox}[title=Mem-Optimizer Prompt: Success Case (Training \& Evaluation), colbacktitle=blue!70!black]
    \textbf{\# System Prompt} \\
    \textbf{\# Role} \\
    You are an expert \textbf{Experience Summarizer} for a memory bank. Your job is to convert one episode into a reusable, general experience. \\
    \textbf{\# Input Data} \\
    1. \textbf{User Query}: The problem context. 2. \textbf{Past Experiences}: Existing rules (indexed as [0], [1]...). 3. \textbf{Model Execution}: The reasoning process. 4. \textbf{Execution Status}: Success or Failure. \\
    \textbf{\# CRITICAL CONSTRAINTS} \\
    1. \textbf{NO ANSWER LEAKAGE}: Never mention specific option indices or answer strings. \\
    2. \textbf{NO SPECIFICS}: Remove specific numbers/names. Replace with variables/concepts. \\
    3. \textbf{NO HALLUCINATION}: Do not invent facts.

    \smallskip
    \textbf{\# ACTION GUIDELINES} \\
    \begin{tcolorbox}[colback=blue!5, colframe=blue!30, boxrule=0.5pt, sharp corners, left=3pt, right=3pt, top=3pt, bottom=3pt]
    \textcolor{blue!80!black}{\textbf{\#\#\# SCENARIO: SUCCESSFUL EXECUTION}} \\
    Extract the underlying \textbf{Truth} or \textbf{Method}. \\
    - \textbf{Content}: Abstract the logic. If knowledge, extract the core fact. \\
    - \textbf{Constraints}: NO specific options (A/B) or specific numbers/entities.
    \end{tcolorbox}

    \smallskip
    \textbf{\# MEMORY MANAGEMENT} \\
    Compare the new insight with \textbf{[Past Experiences]}. Briefly determine whether to \textbf{ADD} a new rule or \textbf{UPDATE <index>} (replace an existing one). \\
    \textbf{\# Output Format} \\
    Strictly follow this structure: \texttt{\#\# Analysis} ... \texttt{\#\# Experience <experience> <value>...</value> <operation>ADD or UPDATE index</operation> </experience>}

    \smallskip
    \hrule
    \smallskip
    
    \textbf{\# User Prompt} \\
    \textbf{\# Task Context} \\
    \textbf{[User Query]} \texttt{\{question\} \{choice\_txt\}} \quad \textbf{[Past Experiences]} \texttt{\{memory\_content\}} \\
    \textbf{[Model Execution]} \texttt{\{trajectory\}} \\
    \textbf{[Execution Status]} \colorbox{blue!15}{\textbf{CORRECT (Success)}}

    \smallskip
    \textbf{\# Instruction} \\
    First, analyze the execution and compare with Past Experiences. Then, generate the XML block with the experience value and the operation (ADD or UPDATE <index>).
    \end{tcolorbox}

    \begin{tcolorbox}[title=Mem-Optimizer Prompt: Failure Case (Training \& Evaluation), colbacktitle=red!70!black]
    \textbf{\# System Prompt} \\
    \textbf{\# Role} \\
    You are an expert \textbf{Experience Summarizer} for a memory bank. Your job is to convert one episode into a reusable, general experience. \\
    \textbf{\# Input Data} \\
    1. \textbf{User Query}: The problem context. 2. \textbf{Past Experiences}: Existing rules (indexed as [0], [1]...). 3. \textbf{Model Execution}: The reasoning process. 4. \textbf{Execution Status}: Success or Failure. \\
    \textbf{\# CRITICAL CONSTRAINTS} \\
    1. \textbf{NO ANSWER LEAKAGE}: Never mention specific option indices or answer strings. \\
    2. \textbf{NO SPECIFICS}: Remove specific numbers/names. Replace with variables/concepts. \\
    3. \textbf{NO HALLUCINATION}: Do not invent facts.

    \smallskip
    \textbf{\# ACTION GUIDELINES} \\
    \begin{tcolorbox}[colback=red!5, colframe=red!30, boxrule=0.5pt, sharp corners, left=3pt, right=3pt, top=3pt, bottom=3pt]
    \textcolor{red!80!black}{\textbf{\#\#\# SCENARIO: FAILED EXECUTION}} \\
    Analyze the \textbf{Root Cause} of the error. \\
    - \textbf{Content}: Identify the \textit{type} of confusion or trap. \\
    - \textbf{Constraints}: DO NOT simply say ``Don't choose X''. DO NOT quote the wrong text as a rule.
    \end{tcolorbox}

    \smallskip
    \textbf{\# MEMORY MANAGEMENT} \\
    Compare the new insight with \textbf{[Past Experiences]}. Briefly determine whether to \textbf{ADD} a new rule or \textbf{UPDATE <index>} (replace an existing one). \\
    \textbf{\# Output Format} \\
    Strictly follow this structure: \texttt{\#\# Analysis} ... \texttt{\#\# Experience <experience> <value>...</value> <operation>ADD or UPDATE index</operation> </experience>}

    \smallskip
    \hrule
    \smallskip
    
    \textbf{\# User Prompt} \\
    \textbf{\# Task Context} \\
    \textbf{[User Query]} \texttt{\{question\} \{choice\_txt\}} \quad \textbf{[Past Experiences]} \texttt{\{memory\_content\}} \\
    \textbf{[Model Execution]} \texttt{\{trajectory\}} \\
    \textbf{[Execution Status]} \colorbox{red!15}{\textbf{INCORRECT (Failure)}}

    \smallskip
    \textbf{\# Instruction} \\
    First, analyze the execution and compare with Past Experiences. Then, generate the XML block with the experience value and the operation (ADD or UPDATE <index>).
    \end{tcolorbox}
\end{tcbraster}

\caption{Comparison of Mem-Optimizer prompt templates for successful (top) and failed (bottom) executions. These templates are employed during both training and evaluation phases to either extract general methodologies or diagnose root causes.}
\label{fig:full_summarizer_comparison}
\end{figure*}
\begin{promptbox}{Executor Prompt Template: Mathematical Reasoning}{prompt:math}
\roleSystem
\texttt{\# Role}\\
You are an expert Math Task Execution Agent. Your goal is to solve mathematical problems by applying logic and methods from \textbf{Past Effective Experiences}.

\texttt{\# Input Data}\\
1. \textbf{Past Experiences}: Relevant formulas, theorems, or similar solved examples.\\
2. \textbf{Question}: The specific math problem you need to solve.

\texttt{\# Instructions}\\
1. Analyze the \textbf{Question} to identify the mathematical concepts involved.\\
2. Refer to the \textbf{Past Experiences} to find the correct formula, method, or logic pattern.\\
3. Perform the \textbf{Problem Solving Process} step-by-step. Show your work, calculations, and derivations clearly.

\tcblower

\roleUser
\texttt{\# Current Task}\\
Solve the following problem.

\textbf{Question}:\\
\promptVar{question}

\textbf{Past Experiences}:\\
\promptVar{memory\_section}

Please reason step by step, and put your final answer within \texttt{\textbackslash boxed\{\}}.
\end{promptbox}
\begin{promptbox}{Executor Unified Prompt Template: Multiple Choice Tasks}{prompt:mcq}
\roleSystem
\texttt{\# Role}\\
You are an expert Task Execution Agent. Your goal is to solve multiple-choice questions by applying \textbf{Past Effective Experiences}.

\texttt{\# Input Data}\\
1. \textbf{Past Experiences}: Historical context or rules to guide your decision.\\
2. \textbf{Question}: The specific problem you need to solve.\\
3. \textbf{Options}: A list of candidate answers.

\texttt{\# Instructions}\\
1. Analyze the \textbf{Question} carefully.\\
2. Refer to the \textbf{Past Experiences} to find the logic or evidence required to solve the problem.\\
3. Evaluate the \textbf{Options} and select the best one.\\
4. \textbf{CRITICAL}:
\begin{itemize}
    \item \textit{[For Index Tasks]:} Identify the \textbf{Index} of the selected option based on a \textbf{0-based system} (i.e., 0, 1, ...).
    \item \textit{[For Letter Tasks]:} Identify the \textbf{Letter} of the selected option (i.e., A, B, C, or D).
\end{itemize}

\tcblower

\roleUser
\texttt{\# Current Task}\\
\textbf{Question}:\\
\promptVar{question}

\textbf{Options}:\\
\promptVar{choice\_block}

\textbf{Past Experiences}:\\
\promptVar{memory\_section}

\texttt{\# Output Format}\\
Analyze the options and the question step-by-step.\\
Output the final answer
\begin{itemize}
    \item \textit{[For Index Tasks]:} index wrapped in \texttt{\textbackslash boxed\{index\}}, e.g., \texttt{\textbackslash boxed\{0\}}.
    \item \textit{[For Letter Tasks]:} single letter wrapped in \texttt{\textbackslash boxed\{Letter\}}, e.g., \texttt{\textbackslash boxed\{A\}}.
\end{itemize}
\end{promptbox}
\begin{promptbox}{Executor Prompt Template: Question Answering (QA)}{prompt:qa}
\roleSystem
\texttt{\# Role}\\
You are an expert Question Answering Agent. Your goal is to answer questions based on the provided \textbf{Context} and applying \textbf{Past Effective Experiences}.

\texttt{\# Input Data}\\
1. \textbf{Past Experiences}: Historical context, strategies, or rules to guide your reasoning.\\
2. \textbf{Context}: Background information, documents, or text passages relevant to the question.\\
3. \textbf{Question}: The specific inquiry you need to answer.

\texttt{\# Instructions}\\
1. Read the \textbf{Context} carefully to extract relevant facts.\\
2. Refer to \textbf{Past Experiences} to find successful reasoning patterns or specific knowledge that supplements the context.\\
3. Synthesize the information to answer the \textbf{Question} accurately and concisely.

\tcblower

\roleUser
\texttt{\# Current Task}\\
\textbf{Question}:\\
\promptVar{question}\\
\promptVar{context\_block}

\textbf{Past Experiences}:\\
\promptVar{memory\_section}

\texttt{\# Output Format}\\
You must strictly follow this format:\\
First, provide your reasoning process, citing the context or experiences where applicable.\\
Then, output the final answer wrapped in \texttt{\textbackslash boxed\{\}}.
\end{promptbox}

\section{Case Study}
This case study \ref{fig:case_study} illustrates how retrieved experiences enable effective knowledge transfer and task completion. The task ``put a clean cloth in countertop'' contains an implicit requirement: the cloth must be \textit{cleaned} before placement, not merely moved. 

\textbf{UMEM Enhanced Agent.} By retrieving experiences from analogous tasks (cleaning plates, knives, and pans), the agent recognizes a generalizable pattern: \textit{locate object $\rightarrow$ pick up $\rightarrow$ go to sinkbasin $\rightarrow$ clean with sinkbasin $\rightarrow$ place on target}. Although the agent initially explores incorrect locations (handtowelholder) and picks up the wrong object (handtowel), it self-corrects upon discovering the cloth and successfully applies the cleaning procedure learned from memory. This demonstrates the agent's ability to \textbf{transfer procedural knowledge} across different object types (plate/knife/pan $\rightarrow$ cloth) and \textbf{recover from exploration errors} through experience-guided reasoning.

\textbf{Baseline.} Lacking prior experiences, this agent interprets the task literally as a simple pick-and-place operation. Despite locating the cloth quickly, it repeatedly executes \texttt{take} $\rightarrow$ \texttt{move} actions without ever invoking the \texttt{clean} command. Notably, even after querying the \texttt{help} command and seeing ``clean (object) with (receptacle)'' in the available actions, the agent fails to connect this capability to the task requirement. This reveals a critical limitation: \textbf{without experiential knowledge linking the task semantics to the required action sequence, the agent cannot infer the missing step}, resulting in an ineffective loop of 30 repeated attempts.

\textbf{Key Insights.} (1) \textit{Semantic understanding}: Experiences provide crucial context for interpreting implicit task requirements (``clean'' as a prerequisite, not just a descriptor). (2) \textit{Efficiency}: With memories extracted by UMEM, the exactor completes the task in 13 steps through meaningful exploration, whereas the baseline agent falls into a futile loop of repetitive actions and exhausts 30 steps without solving the task. (3) \textit{Generalization}: Experiences about cleaning plates/knives/pans successfully transfer to cleaning cloth, demonstrating cross-object procedural generalization.

\begin{figure*}[t]
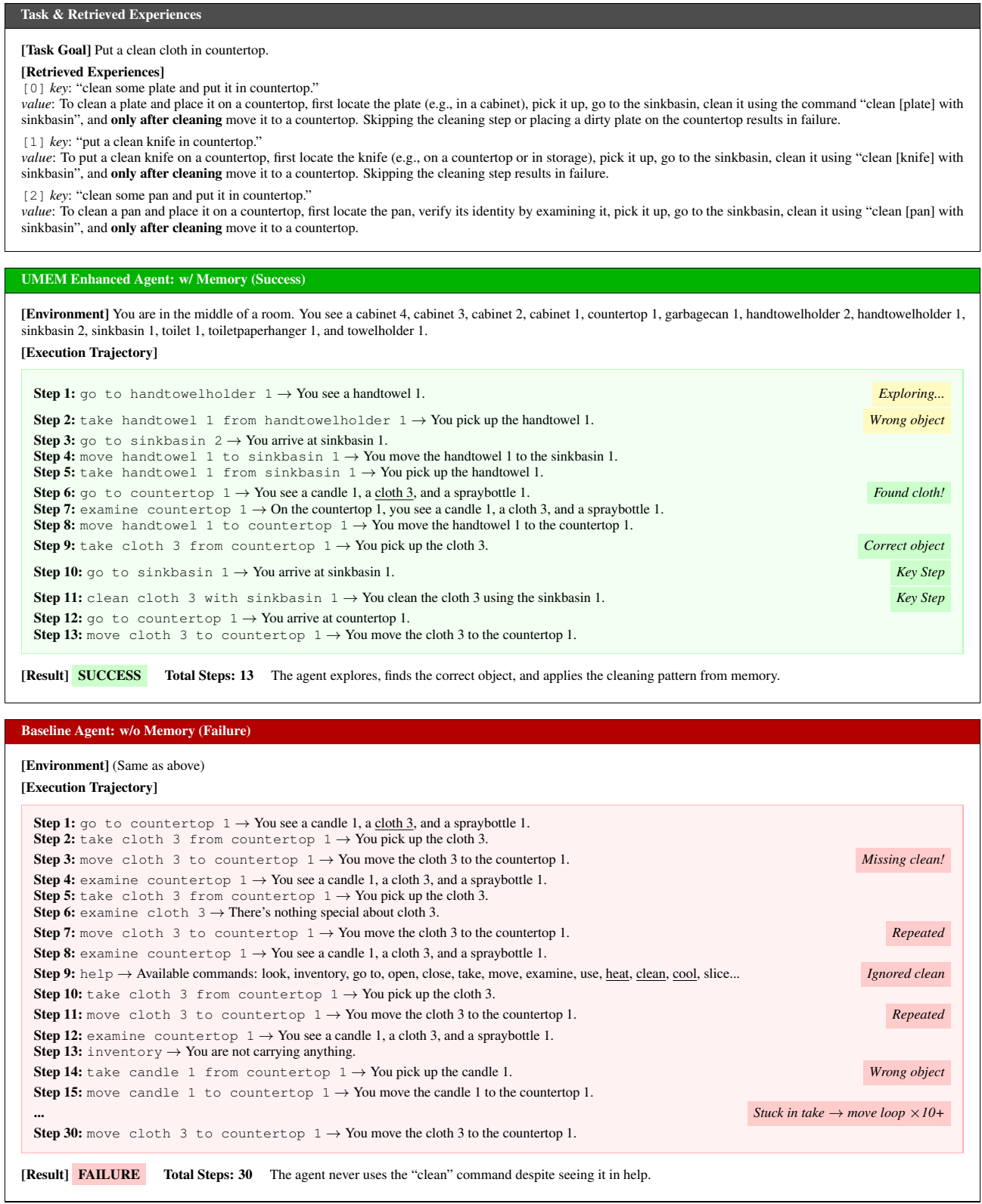

    \centering
    \scriptsize
    \begin{tcbraster}[
        raster columns=1, 
        raster row skip=8pt, 
        sharp corners, 
        boxrule=0.5pt,
        colback=white,
        colframe=black,
        fonttitle=\bfseries,
        left=5pt, right=5pt, top=5pt, bottom=5pt
    ]
        \begin{tcolorbox}[title=Task \& Retrieved Experiences, colbacktitle=gray!60!black]
        \textbf{[Task Goal]} Put a clean cloth in countertop.

        \smallskip
        \textbf{[Retrieved Experiences]} \\
        \texttt{[0]} \textit{key}: ``clean some plate and put it in countertop.'' \\
        \textit{value}: To clean a plate and place it on a countertop, first locate the plate (e.g., in a cabinet), pick it up, go to the sinkbasin, clean it using the command ``clean [plate] with sinkbasin'', and \textbf{only after cleaning} move it to a countertop. Skipping the cleaning step or placing a dirty plate on the countertop results in failure. \\[3pt]
        \texttt{[1]} \textit{key}: ``put a clean knife in countertop.'' \\
        \textit{value}: To put a clean knife on a countertop, first locate the knife (e.g., on a countertop or in storage), pick it up, go to the sinkbasin, clean it using ``clean [knife] with sinkbasin'', and \textbf{only after cleaning} move it to a countertop. Skipping the cleaning step results in failure. \\[3pt]
        \texttt{[2]} \textit{key}: ``clean some pan and put it in countertop.'' \\
        \textit{value}: To clean a pan and place it on a countertop, first locate the pan, verify its identity by examining it, pick it up, go to the sinkbasin, clean it using ``clean [pan] with sinkbasin'', and \textbf{only after cleaning} move it to a countertop.
        \end{tcolorbox}

        \begin{tcolorbox}[title=UMEM Enhanced Agent: w/ Memory (Success), colbacktitle=green!70!black]
        \textbf{[Environment]} You are in the middle of a room. You see a cabinet 4, cabinet 3, cabinet 2, cabinet 1, countertop 1, garbagecan 1, handtowelholder 2, handtowelholder 1, sinkbasin 2, sinkbasin 1, toilet 1, toiletpaperhanger 1, and towelholder 1.

        \smallskip
        \textbf{[Execution Trajectory]}
        \begin{tcolorbox}[colback=green!5, colframe=green!30, boxrule=0.5pt, sharp corners, left=3pt, right=3pt, top=3pt, bottom=3pt]
        \textbf{Step 1:} \texttt{go to handtowelholder 1} $\rightarrow$ You see a handtowel 1. \hfill \colorbox{yellow!30}{\textit{Exploring...}} \\
        \textbf{Step 2:} \texttt{take handtowel 1 from handtowelholder 1} $\rightarrow$ You pick up the handtowel 1. \hfill \colorbox{yellow!30}{\textit{Wrong object}} \\
        \textbf{Step 3:} \texttt{go to sinkbasin 2} $\rightarrow$ You arrive at sinkbasin 1. \\
        \textbf{Step 4:} \texttt{move handtowel 1 to sinkbasin 1} $\rightarrow$ You move the handtowel 1 to the sinkbasin 1. \\
        \textbf{Step 5:} \texttt{take handtowel 1 from sinkbasin 1} $\rightarrow$ You pick up the handtowel 1. \\
        \textbf{Step 6:} \texttt{go to countertop 1} $\rightarrow$ You see a candle 1, a \underline{cloth 3}, and a spraybottle 1. \hfill \colorbox{green!20}{\textit{Found cloth!}} \\
        \textbf{Step 7:} \texttt{examine countertop 1} $\rightarrow$ On the countertop 1, you see a candle 1, a cloth 3, and a spraybottle 1. \\
        \textbf{Step 8:} \texttt{move handtowel 1 to countertop 1} $\rightarrow$ You move the handtowel 1 to the countertop 1. \\
        \textbf{Step 9:} \texttt{take cloth 3 from countertop 1} $\rightarrow$ You pick up the cloth 3. \hfill \colorbox{green!20}{\textit{Correct object}} \\
        \textbf{Step 10:} \texttt{go to sinkbasin 1} $\rightarrow$ You arrive at sinkbasin 1. \hfill \colorbox{green!20}{\textit{Key Step}} \\
        \textbf{Step 11:} \texttt{clean cloth 3 with sinkbasin 1} $\rightarrow$ You clean the cloth 3 using the sinkbasin 1. \hfill \colorbox{green!20}{\textit{Key Step}} \\
        \textbf{Step 12:} \texttt{go to countertop 1} $\rightarrow$ You arrive at countertop 1. \\
        \textbf{Step 13:} \texttt{move cloth 3 to countertop 1} $\rightarrow$ You move the cloth 3 to the countertop 1.
        \end{tcolorbox}
        
        \textbf{[Result]} \colorbox{green!20}{\textbf{SUCCESS}} \quad \textbf{Total Steps: 13} \quad The agent explores, finds the correct object, and applies the cleaning pattern from memory.
        \end{tcolorbox}

        \begin{tcolorbox}[title=Baseline Agent: w/o Memory (Failure), colbacktitle=red!70!black]
        \textbf{[Environment]} (Same as above)

        \smallskip
        \textbf{[Execution Trajectory]}
        \begin{tcolorbox}[colback=red!5, colframe=red!30, boxrule=0.5pt, sharp corners, left=3pt, right=3pt, top=3pt, bottom=3pt]
        \textbf{Step 1:} \texttt{go to countertop 1} $\rightarrow$ You see a candle 1, a \underline{cloth 3}, and a spraybottle 1. \\
        \textbf{Step 2:} \texttt{take cloth 3 from countertop 1} $\rightarrow$ You pick up the cloth 3. \\
        \textbf{Step 3:} \texttt{move cloth 3 to countertop 1} $\rightarrow$ You move the cloth 3 to the countertop 1. \hfill \colorbox{red!20}{\textit{Missing clean!}} \\
        \textbf{Step 4:} \texttt{examine countertop 1} $\rightarrow$ You see a candle 1, a cloth 3, and a spraybottle 1. \\
        \textbf{Step 5:} \texttt{take cloth 3 from countertop 1} $\rightarrow$ You pick up the cloth 3. \\
        \textbf{Step 6:} \texttt{examine cloth 3} $\rightarrow$ There's nothing special about cloth 3. \\
        \textbf{Step 7:} \texttt{move cloth 3 to countertop 1} $\rightarrow$ You move the cloth 3 to the countertop 1. \hfill \colorbox{red!20}{\textit{Repeated}} \\
        \textbf{Step 8:} \texttt{examine countertop 1} $\rightarrow$ You see a candle 1, a cloth 3, and a spraybottle 1. \\
        \textbf{Step 9:} \texttt{help} $\rightarrow$ Available commands: look, inventory, go to, open, close, take, move, examine, use, \underline{heat}, \underline{clean}, \underline{cool}, slice... \hfill \colorbox{red!20}{\textit{Ignored clean}} \\
        \textbf{Step 10:} \texttt{take cloth 3 from countertop 1} $\rightarrow$ You pick up the cloth 3. \\
        \textbf{Step 11:} \texttt{move cloth 3 to countertop 1} $\rightarrow$ You move the cloth 3 to the countertop 1. \hfill \colorbox{red!20}{\textit{Repeated}} \\
        \textbf{Step 12:} \texttt{examine countertop 1} $\rightarrow$ You see a candle 1, a cloth 3, and a spraybottle 1. \\
        \textbf{Step 13:} \texttt{inventory} $\rightarrow$ You are not carrying anything. \\
        \textbf{Step 14:} \texttt{take candle 1 from countertop 1} $\rightarrow$ You pick up the candle 1. \hfill \colorbox{red!20}{\textit{Wrong object}} \\
        \textbf{Step 15:} \texttt{move candle 1 to countertop 1} $\rightarrow$ You move the candle 1 to the countertop 1. \\
        \textbf{...} \hfill \colorbox{red!20}{\textit{Stuck in take $\rightarrow$ move loop $\times$10+}} \\
        \textbf{Step 30:} \texttt{move cloth 3 to countertop 1} $\rightarrow$ You move the cloth 3 to the countertop 1.
        \end{tcolorbox}
        
        \textbf{[Result]} \colorbox{red!20}{\textbf{FAILURE}} \quad \textbf{Total Steps: 30} \quad The agent never uses the ``clean'' command despite seeing it in help.
        \end{tcolorbox}
    \end{tcbraster}

    \caption{Case study comparing UMEM enhanced agent and baseline.}
    \label{fig:case_study}
\end{figure*}

\section{Procedure for Evolutionary Memory Management}
\label{app:mem_mgmt}
\begin{algorithm}[t]
\small
\caption{UMEM Training: Semantic Neighborhood Modeling and GRPO}
\label{alg:umem}\KwIn{Query corpus $\mathcal{D}$, frozen executor $\mathcal{E}$, Mem-Optimizer $\pi_{\phi}$, Neighborhood size $N$, Group size $G$}
\KwOut{Trained parameters $\phi$ and evolved memory bank $\mathcal{B}$}\BlankLine
\textbf{Phase 1: Offline Semantic Neighborhood Modeling};
\ForEach{$q \in \mathcal{D}$}{
$\mathcal{N}_N(q) \leftarrow \text{Retrieve } N \text{ nearest neighbors for } q \text{ from } \mathcal{D} \setminus \{q\}$;
}\BlankLine
\textbf{Phase 2: GRPO-based Online Memory Evolution};
\For{each training step}{
Sample a mini-batch $\mathbf{Q} \subset \mathcal{D}$;
\ForEach{$q \in \mathbf{Q}$}{
$\tau_q \leftarrow \mathcal{E}(q, \mathcal{B})$;
\For{$g \leftarrow 1$ \KwTo $G$}{
$o^{(g)} \sim \pi_{\phi}(\cdot \mid q, \tau_q, \mathcal{B})$;
$r_f^{(g)} \leftarrow \mathbb{I}[\text{FormatOK}(o^{(g)})]$;
$\tilde{\mathcal{B}}^{(g)} \leftarrow \text{Apply } o^{(g)} \text{ to } \mathcal{B}$;
$r_g^{(g)} \leftarrow \frac{1}{|\mathcal{N}_N(q)|} \sum_{q' \in \mathcal{N}_N(q)} \text{UtilityGain}(q', \tilde{\mathcal{B}}^{(g)}, \mathcal{B})$;
$r^{(g)} \leftarrow r_f^{(g)} + r_g^{(g)}$;
}
Update $\phi$ via GRPO using group advantages $\{r^{(g)} - \text{mean}(r)\}_{g=1}^G$;
$\mathcal{B} \leftarrow \tilde{\mathcal{B}}^{(g^\star)}$ where $g^\star = \arg\max_g r^{(g)}$;
}
}
\Return{$\phi, \mathcal{B}$};
\end{algorithm}
Algorithm \ref{alg:umem} details the training process of UMEM, characterized by the co-evolution of the Mem-Optimizer $\pi_\phi$ and the memory bank $\mathcal{B}$. Prior to training, we perform \textbf{Semantic Neighborhood Modeling} to identify $\mathcal{N}_N(q)$ for each query $q$ based on embedding similarity, preventing shortcut learning. The Mem-Optimizer is then optimized through the following iterative stages:\begin{itemize}\item \textbf{(1) Memory-Augmented Execution}: The frozen executor $\mathcal{E}$ performs task $q$ using retrieved context from the current memory $\mathcal{B}$ to generate an initial trajectory $\tau_q$.\item \textbf{(2) Policy Rollout}: The Mem-Optimizer $\pi_\phi$ samples a group of $G$ candidate operations $\{o^{(g)}\}_{g=1}^G$ (e.g., \textsc{Add} or \textsc{Update}) based on $q$, $\tau_q$, and the retrieved memory.\item \textbf{(3) Marginal Utility Reward}: For each rollout, we compute a format reward $r_f$ for structural correctness and a marginal utility reward $r_g$, defined as the average performance gain (success rate and efficiency) across the semantic neighborhood $\mathcal{N}_N(q)$.\item \textbf{(4) Optimization via GRPO}: The policy $\pi_\phi$ is updated using group-relative advantages derived from the combined rewards $r_f + r_g$, facilitating stable policy refinement without a critic network.\item \textbf{(5) Online Memory Evolution}: The memory bank $\mathcal{B}$ is updated by committing the best-performing operation $o^{(g^\star)}$ from the group, ensuring the knowledge base evolves alongside the policy.\end{itemize}
\end{document}